\documentclass{article} 
\usepackage[preprint]{colm2026_conference}

\usepackage{microtype}
\usepackage{hyperref}
\usepackage{url}
\usepackage{booktabs}
\usepackage{booktabs} 
\usepackage[most]{tcolorbox}
\usepackage{multirow} 
\usepackage[table]{xcolor} 
\usepackage{caption}
\usepackage{colortbl}
\usepackage{xcolor}
\usepackage{listings}
\usepackage{xcolor}
\usepackage{float}

\definecolor{codegreen}{rgb}{0,0.6,0}
\definecolor{codegray}{rgb}{0.5,0.5,0.5}
\definecolor{codepurple}{rgb}{0.58,0,0.82}
\definecolor{backcolour}{rgb}{0.95,0.95,0.92}
\usepackage{colortbl}
\usepackage{xcolor}

\definecolor{high_align}{HTML}{D1E8FF} 
\definecolor{low_align}{HTML}{FADBD8}  
\definecolor{misaligned}{HTML}{F1948A} 
\definecolor{high_perf}{HTML}{D5F5E3}  
\definecolor{categorygray}{gray}{0.92}
\newtcolorbox{dtsexample}[1]{
    colback=white,
    colframe=gray!75!black,
    fonttitle=\bfseries,
    title=#1,
    enhanced,
    attach boxed title to top left={yshift=-2mm, xshift=2mm},
    boxed title style={colback=gray!75!black},
    sharp corners,
    boxrule=0.5pt,
    bottomrule=2pt, 
}

 \usepackage{multirow} 
\usepackage{lineno}
\definecolor{myorange}{RGB}{37, 79, 150}
\usepackage{microtype}
\usepackage{mathtools,amssymb,amsthm}

\usepackage{booktabs}
\usepackage{caption}
\usepackage{float}
\usepackage{algpseudocode}
\usepackage{algorithm}
\usepackage{amsmath,amssymb,amsthm}
\usepackage{algorithm}
\usepackage{algpseudocode}
\usepackage[most]{tcolorbox}
\tcbuselibrary{skins,breakable}
\usepackage{booktabs, xcolor, colortbl, graphicx}
\usepackage{xcolor}
\usepackage{listings}
\usepackage{booktabs}
\usepackage[table]{xcolor}
\usepackage{siunitx}
\usepackage{booktabs}
\usepackage[table]{xcolor}
\usepackage{tabularx}
\usepackage{siunitx} 
\usepackage{booktabs}
\usepackage{pifont}
\usepackage{xcolor}
\usepackage{wrapfig}
\usepackage{enumitem}
\newcommand{\llmicon}[1]{\raisebox{-0.15\height}{\includegraphics[height=8pt]{#1}}}
\newcommand{\cmark}{\textcolor{green!60!black}{\ding{51}}}%
\newcommand{\xmark}{\textcolor{red!70!black}{\ding{55}}}%

\usepackage{microtype}     
\usepackage{siunitx}       

\usepackage{amsthm}
\usepackage{amsmath} 
\usepackage{algorithm}
\usepackage{algpseudocode} 

\newtheorem{theorem}{Theorem}[section]

\definecolor{codebg}{gray}{0.95}
\definecolor{keyword}{RGB}{0,0,180}
\definecolor{comment}{RGB}{0,100,0}
\definecolor{string}{RGB}{160,0,0}

\lstdefinestyle{pythonstyle}{
  backgroundcolor=\color{codebg},
  basicstyle=\ttfamily\small,
  commentstyle=\color{comment}\itshape,
  keywordstyle=\color{keyword}\bfseries,
  stringstyle=\color{string},
  numbers=left,
  numberstyle=\tiny,
  numbersep=8pt,
  breaklines=true,
  frame=single,
  rulecolor=\color{orange!50!black},
  frameround=tttt,
  tabsize=4,
  showstringspaces=false
}
\usepackage{graphicx} 

\usepackage[utf8]{inputenc} 
\usepackage[T1]{fontenc}    
\usepackage{hyperref}       
\usepackage{url}            
\usepackage{booktabs}       
\usepackage{amsfonts}       
\usepackage{nicefrac}       
\usepackage{microtype}      
\usepackage{xcolor}         

\definecolor{darkblue}{rgb}{0, 0, 0.5}
\hypersetup{colorlinks=true, citecolor=darkblue, linkcolor=darkblue, urlcolor=darkblue}

\usepackage{xcolor, soul}
\sethlcolor{pink}

\usepackage{subcaption}

\title{\raisebox{-0.35\height}{\includegraphics[height=2.2em]{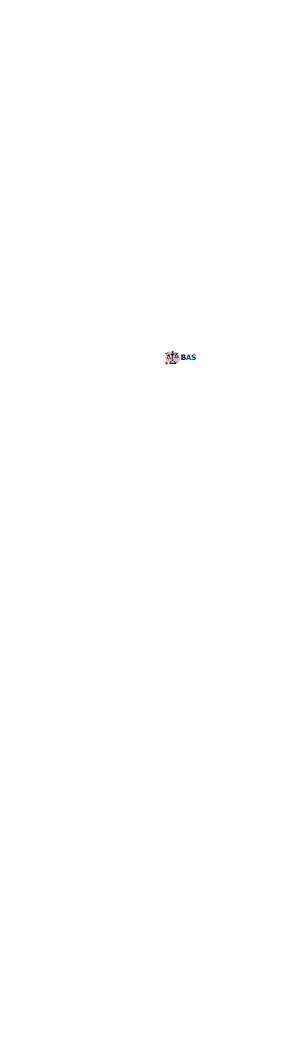}}\,: 
\textbf{A Decision-Theoretic Approach to Evaluating Large Language Model Confidence}}

\author{Sean Wu$^{*1}$, Fredrik K. Gustafsson$^{*1}$, Edward Phillips$^{1}$, Boyan Gao$^{1}$, Anshul Thakur$^{1}$,\\ \textbf{David A. Clifton$^{1,2}$}\\
\\
$^{1}$Department of Engineering Science, University of Oxford\\
$^{2}$Oxford Suzhou Centre for Advanced Research\\
\small$^{*}$Equal Contribution\\
\small\texttt{\{sean.wu, fredrik.gustafsson, anshul.thakur, david.clifton\}@eng.ox.ac.uk}
}

\begin{document}

\ifcolmsubmission
\linenumbers
\fi

\maketitle

\begin{abstract}
Large language models (LLMs) often produce confident but incorrect answers in settings where abstention would be safer. Standard evaluation protocols, however, require a response and do not account for how confidence should guide decisions under different risk preferences. To address this gap, we introduce the \emph{Behavioral Alignment Score} (BAS), a decision-theoretic metric for evaluating how well LLM confidence supports abstention-aware decision making. BAS is derived from an explicit answer-or-abstain utility model and aggregates realized utility across a continuum of risk thresholds, yielding a measure of decision-level reliability that depends on both the magnitude and ordering of confidence. We show theoretically that truthful confidence estimates uniquely maximize expected BAS utility, linking calibration to decision-optimal behavior. BAS is related to proper scoring rules such as log loss, but differs structurally: log loss penalizes underconfidence and overconfidence symmetrically, whereas BAS imposes an asymmetric penalty that strongly prioritizes avoiding overconfident errors. Using BAS alongside widely used metrics such as ECE and AURC, we then construct a benchmark of self-reported confidence reliability across multiple LLMs and tasks. Our results reveal substantial variation in decision-useful confidence, and while larger and more accurate models tend to achieve higher BAS, even frontier models remain prone to severe overconfidence. Importantly, models with similar ECE or AURC can exhibit very different BAS due to highly overconfident errors, highlighting limitations of standard metrics. We further show that simple interventions, such as top-$k$ confidence elicitation and post-hoc calibration, can meaningfully improve confidence reliability. Overall, our work provides both a principled metric and a comprehensive benchmark for evaluating LLM confidence reliability\footnote{Code is available at \url{https://github.com/SeanWu25/Behavioral-Alignment-Score}.}.
\end{abstract}

\section{Introduction}
\label{sec:intro}

One of the grand challenges that large language models (LLMs) face today is hallucination, where models generate responses that sound plausible but are factually incorrect~\citep{Huang_2025, simhi2025trust}. This phenomenon persists even in frontier LLMs such as GPT-5 and DeepSeek \citep{zhang2023complete, guo2025deepseek}, making their deployment in high-stakes domains such as healthcare, science, and law potentially unsafe \citep{jones2025ai, sun2024ai, wu2024benchmarking}. A growing body of research suggests that hallucinations are not simply artifacts of modeling, but are also driven by how LLMs are evaluated and by limitations in post-training alignment \citep{kalai2025language, dang2025survey, yao2025reasoning}. Current evaluation metrics and benchmarks often reward models for producing answers even when they are uncertain, rather than for expressing uncertainty in a way that supports reliable decision making \citep{madhusudhan2024llms}. In many real-world settings, however, the appropriate action depends on the model's uncertainty. For example, a clinician consulting a medical LLM (e.g., Med-PaLM, Meditron, or MedGemma \citep{tu2024towards, chen2023meditron, sellergren2025medgemma}) would benefit from a system whose predictions reflect not only what it believes, but also how confident it is. This mirrors how a risk-aware human makes decisions, acting only when their confidence outweighs the potential cost of error.

\begin{figure}[t]
    \centering
\includegraphics[width=\textwidth]{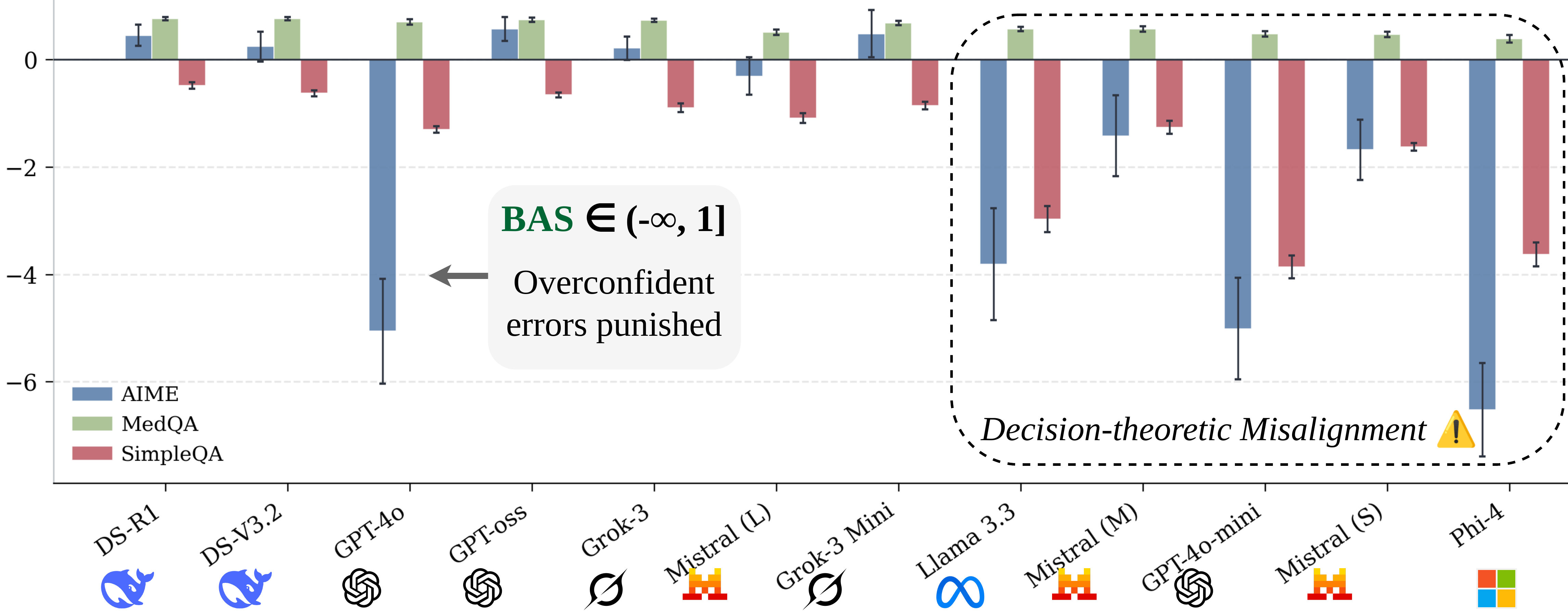}\vspace{-2.0mm}
\caption{We introduce the Behavioral Alignment Score (BAS), a decision-theoretic metric derived from an explicit answer-or-abstain utility model, and use it to evaluate LLM confidence reliability across a diverse set of models and tasks. Reliability varies substantially across tasks, and while larger and more accurate models tend to also achieve higher BAS, even frontier models remain prone to overconfidence on complex, open-ended tasks.}
 \label{fig:overview}
\end{figure}

LLMs do however lack a single reliable measure of uncertainty and may express it through diverse and imperfect signals \citep{beigi2024rethinkinguncertaintycriticalreview, kapoor2024large}. Several proxies exist, including token-level likelihoods \citep{zhang2025token}, entropy over sampled responses \citep{farquhar2024detecting}, semantic dispersion \citep{phillips2025geometric, li2025semanticvolumequantifyingdetecting}, and confidence estimated through prompting or self-reflection \citep{taubenfeld2025confidence}, but these approaches often produce inconsistent results \citep{ma2025estimating}. Many methods also rely on internal activations or repeated sampling, which are infeasible for proprietary models such as GPT-5 or Claude \citep{phillips2025geometric}. Evaluating model correctness is also noisy and domain dependent, especially in open-ended tasks such as clinical decision making or multi-step reasoning \citep{gaber2025evaluating, gumilar2024assessment, hager2024evaluation}.

These limitations motivate a practical question: \emph{can we trust the self-reported confidence of LLMs for decision making under uncertainty?} In particular, we consider the setting where only text-level access is available, and models are prompted to produce both an answer and an associated confidence estimate. In this setting, the usefulness of confidence depends not only on calibration or ranking, but on whether it supports reliable answer-or-abstain decisions under varying risk preferences.

To address this, we introduce the Behavioral Alignment Score (BAS), a decision-theoretic metric for evaluating how well a model's reported confidence supports abstention-aware decision making. BAS assigns realized utility to model predictions under an explicit cost model, where correct answers are rewarded, incorrect answers are penalized asymmetrically based on confidence, and abstention yields zero utility. It aggregates performance across a continuum of risk thresholds, yielding a measure of decision-level reliability that depends on both the magnitude and ordering of reported confidence. We show theoretically that truthful confidence estimates uniquely maximize expected BAS utility, thereby linking calibration to decision-optimal behavior. BAS is most naturally understood as a decision-theoretic proper score for selective prediction. In this sense, it is related to standard proper scoring rules such as log loss and Brier score, which also reward truthful confidence reporting. However, these metrics penalize underconfidence and overconfidence symmetrically, whereas BAS imposes an asymmetric penalty that strongly prioritizes avoiding overconfident errors. More fundamentally, BAS differs in its objective: it is derived from an explicit abstention-aware utility model, rather than from generic probability forecasting.

Beyond the BAS metric itself, we construct a comprehensive benchmark of LLM self-reported confidence (Figure~\ref{fig:overview}). We evaluate a diverse set of open- and closed-source models across multiple domains, including mathematical reasoning, medical question answering, and factual QA, and compare BAS with widely used reliability metrics such as ECE and AURC. Our results reveal substantial variation in confidence reliability, and while larger and more accurate models tend to also achieve better reliability metrics, even frontier models remain prone to severe overconfidence in challenging settings. Importantly, models with similar ECE or AURC can exhibit substantially different BAS due to highly overconfident errors, highlighting limitations of standard metrics. We further show that confidence elicitation strategies and post-hoc calibration can significantly affect confidence reliability.

Our main contributions are the following:
\begin{itemize}
     \item \textbf{A decision-theoretic metric for confidence evaluation.} We introduce BAS, a confidence reliability metric derived from an explicit answer-or-abstain utility model, and show that it defines a proper scoring objective for selective prediction.
    
    \item \textbf{A benchmark of LLM confidence.} We construct a comprehensive benchmark of self-reported LLM confidence across a diverse set of models, tasks and domains, and evaluate BAS alongside widely used metrics such as ECE and AURC.
    
    \item \textbf{Empirical analysis of confidence reliability.} We show that confidence reliability varies substantially across tasks and generally improves both with model scale and predictive performance, but that different metrics can yield different conclusions. In particular, models with similar ECE or AURC can exhibit markedly different decision-level reliability under BAS due to rare but highly overconfident errors.
    
    \item \textbf{Practical insights for improving reliability.} We further show that confidence reliability depends not only on the model, but also on how confidence is elicited and calibrated. In particular, simple strategies such as top-$k$ elicitation and post-hoc calibration can reduce overconfidence and improve decision-level reliability.
\end{itemize}

\section{Method: The Behavioral Alignment Score}
\label{sec:bas}

We propose the \emph{Behavioral Alignment Score} (BAS), a decision-theoretic approach for evaluating how well a language model's reported confidence supports answer-or-abstain decisions under varying risk preferences. The goal of BAS is to measure \emph{decision-level reliability}: whether a model appropriately chooses to answer or abstain based on its confidence. We now present the decision-theoretic formulation underlying BAS, followed by its closed-form definition and its relation to standard proper scoring rules and widely used metrics.

\subsection{Decision-Theoretic Formulation}
We formalize BAS as a selective prediction problem with abstention. Given a prompt $x$, the model produces a response $R$ together with a scalar confidence value $s \in [0, 1)$, intended to reflect the probability that the response is correct. Let $Z \in \{0,1\}$ denote the correctness of the response ($Z=1$ if correct, $0$ otherwise), and let $p = \mathbb{P}(Z=1 \mid x, R)$ denote the true (but unobserved) probability of correctness.

We consider a downstream decision-maker who must choose whether to answer (i.e., output the model response $R$) or abstain based on the reported model confidence $s$. The decision-maker is characterized by a \emph{risk tolerance} threshold $t \in [0,1)$, which determines the trade-off between the benefit of a correct answer and the cost of an incorrect one. We define a selective utility function $S_t$ that assigns unit reward to a correct answer, a risk-dependent penalty $-\frac{t}{1-t}$ to an incorrect answer, and zero utility to abstention:
\begin{equation}
\label{eq:utility}
S_t(Z, a) \;=\;
\begin{cases}
\;\;1, & a=\textsc{Answer},\, Z=1,\\[2pt]
-\frac{t}{1-t}, & a=\textsc{Answer},\, Z=0,\\[2pt]
\;\;0, & a=\textsc{Abstain}.
\end{cases}
\end{equation}
Under this selective utility $S_t$, the decision that maximizes the expected utility, $\mathbb{E}[S_t] = p - (1-p)\frac{t}{1-t}$, is to answer if and only if $p \ge t$. Since the true probability of correctness $p$ is not available at test time, we instead define a model-induced decision policy $\pi_s(t)$ that uses the reported model confidence $s$ as a proxy:
\begin{equation}
\label{eq:policy}
\pi_s(t) \;=\;
\begin{cases}
\textsc{Answer}, & s \ge t,\\
\textsc{Abstain}, & s < t.
\end{cases}
\end{equation}
We evaluate the model by its expected utility under the policy $\pi_s(t)$, taken over both the randomness in correctness $Z$ and the distribution of risk thresholds $t$. Under a uniform distribution over $t \in [0,1)$, \emph{the expected BAS utility} is defined as:
\begin{equation}
\label{eq:expected_bas}
\int_{0}^{1} \mathbb{E}[S_t(Z, \pi_s(t))] \, dt,
\end{equation}
where the inner expectation is taken over $Z$. This formulation raises a key question: when does using $s$ in place of $p$ lead to optimal decisions? The following theorem establishes that, under BAS, truthful confidence estimates are uniquely optimal (the proof is in Appendix~\ref{sec:appendix_bas_theoretical}):

\begin{theorem}[Optimality of BAS Utility]
\label{thm:bas_optimal}
Let $s \in [0, 1)$ be the reported model confidence and $p \in [0, 1]$ the true probability of correctness. The expected BAS utility (Eq.~\ref{eq:expected_bas}) is uniquely maximized when $s = p$ for all $p < 1$. For $p=1$, the expected utility is strictly increasing in $s$ and achieves its supremum as $s \to 1$.
\end{theorem}

This result shows that BAS defines a \emph{proper decision-theoretic scoring objective}: reporting confidence truthfully maximizes the expected utility in Eq.~\ref{eq:expected_bas}, and therefore induces optimal answer-or-abstain behavior. While the standard BAS assumes a uniform prior over risk tolerance for general-purpose evaluation, this formulation naturally generalizes to arbitrary risk profiles via a weighting function $w(t)$ for safety-critical deployments (see Appendix~\ref{sec:weighted_bas_consolidated}).

\subsection{The BAS Metric}
\label{sec:bas_score}
Rather than evaluating performance at a fixed risk threshold, BAS aggregates decision utility across the full spectrum of risk preferences $t \in [0,1)$. We define the primary metric as the expected BAS utility under a uniform distribution over $t$, as defined in Eq.~\ref{eq:expected_bas}.

In practice, we compute the \emph{realized} per-example contribution $U(s, Z)$ and report the dataset mean, which corresponds to this expectation under the empirical distribution. Under the policy $\pi_s$, the model answers if and only if $t \le s$ and abstains otherwise. Since abstention yields zero utility (Eq.~\ref{eq:utility}), the expected BAS utility in Eq.~\ref{eq:expected_bas} reduces to an integral over the answering region, $\int_0^s \mathbb{E}[S_t]\, dt$. Evaluating this integral yields a closed-form expression for the realized utility, which depends on the observed correctness outcome $Z$:
\begin{equation}
\label{eq:bas_closed_form}
U(s, Z) \;=\;
\begin{cases}
s, & Z=1,\\[4pt]
s + \ln(1 - s), & Z=0.
\end{cases}
\end{equation}
Given a dataset of $N$ examples, where each example consists of a model response $R_i$, a confidence value $s_i$, and a corresponding correctness label $Z_i$, the $\mathrm{BAS}$ metric is finally computed according to:
\begin{equation}
\label{eq:BAS_metric}
\mathrm{BAS} = \frac{1}{N}\sum_{i=1}^{N} U(s_i, Z_i).
\end{equation}
Intuitively, correct predictions ($Z=1$) receive utility proportional to the model confidence $s$, while incorrect predictions incur a penalty that grows rapidly as $s \to 1$. In particular, the logarithmic term $\ln(1 - s)$ diverges to $-\infty$ for highly confident errors, ensuring that overconfident hallucinations are penalized severely. Thus, $\mathrm{BAS} \in (-\infty, 1]$, where higher values indicate better decision-level reliability.

\subsection{Relation to Existing Metrics}
\begin{table}[t]
    \centering
	\resizebox{0.80\textwidth}{!}{%
          \begin{tabular}{@{}lcccccc@{}}
            \toprule
            Property 
            & Accuracy 
            & ECE 
            & AURC 
            & Brier Score  
            & Log Loss
            & \textbf{BAS} \\
            \midrule
            Confidence magnitude
              & \xmark & \cmark & \xmark & \cmark & \cmark & \cmark \\
            Calibration
              & \xmark & \cmark & \xmark & \cmark & \cmark & \cmark \\
            Ranking
              & \xmark & \xmark & \cmark & \cmark & \cmark & \cmark \\
            Proper scoring rule
              & \xmark & \xmark & \xmark & \cmark & \cmark & \cmark \\
            Strongly penalizes overconfidence
              & \xmark & \xmark & \xmark & \xmark & \cmark & \cmark \\
            Asymmetric overconfidence penalty\hspace{1.0mm}
              & \xmark & \xmark & \xmark & \xmark & \xmark & \cmark \\
            Decision-theoretic derivation
              & \xmark & \xmark & \xmark & \xmark & \xmark & \cmark \\
            \bottomrule
          \end{tabular}
	}
    \vspace{-1.5mm}
    \caption{Conceptual comparison of BAS with existing evaluation metrics. BAS is most closely related to proper scoring rules such as log loss, sharing a sensitivity to high-confidence errors, but differs in imposing an asymmetric penalty that strongly prioritizes avoiding model overconfidence. In comparison, widely used calibration- and ranking-based metrics such as ECE and AURC capture only partial aspects of confidence reliability.}\vspace{-2.0mm}
    \label{tab:bas-comparison}
\end{table}

We position BAS relative to widely used metrics for confidence reliability evaluation, such as Expected Calibration Error (ECE) and the Area Under the Risk-Coverage Curve (AURC), as well as standard proper scoring rules such as Brier score and log loss. These metrics capture complementary aspects of confidence reliability, but differ in both motivation and interpretation. Table~\ref{tab:bas-comparison} summarizes these relationships.

ECE evaluates probability calibration by measuring whether predicted confidence matches empirical accuracy, but does not account for abstention or the asymmetric consequences of overconfident errors. AURC evaluates ranking quality for selective prediction by measuring the trade-off between risk and coverage, but is insensitive to the absolute magnitude of confidence and therefore does not penalize extreme overconfidence when ranking is preserved. BAS clearly differs from these calibration- and ranking-based metrics. In particular, ECE and AURC fail to capture the impact of rare but highly overconfident errors. In BAS, such errors incur large negative utility due to the logarithmic penalty $\ln(1 - s)$ in Eq.~\ref{eq:bas_closed_form}, directly reflecting their effect on decision-level risk. As shown in Appendix~\ref{sec:appendix_metrics_comparison}, these differences can lead to concrete divergences in practice: models can exhibit identical ECE or AURC while differing substantially in BAS. These examples highlight that calibration or ranking metrics alone may not fully capture decision-level reliability in abstention-aware settings.

Standard proper scoring rules such as the Brier score and log loss provide a closer comparison to BAS. For a binary outcome $Z \in \{0,1\}$ with predicted confidence $s$, the log loss is given by $\ell_{\log}(s, Z) = -\bigl[Z \ln s + (1 - Z) \ln(1 - s)\bigr]$. Like BAS, log loss strongly penalizes overconfident errors via the logarithmic term $\ln(1 - s)$. In contrast, the Brier score applies a bounded quadratic penalty $(s - Z)^2$. Because of the shared sensitivity to highly confident incorrect predictions, BAS is most closely related to log loss among existing metrics. However, an important distinction is that log loss penalizes underconfidence and overconfidence symmetrically, through the terms $\ln s$ and $\ln(1-s)$, respectively. BAS, in contrast, imposes an \emph{asymmetric} penalty: it strongly penalizes overconfident errors while only linearly rewarding correct predictions. As a result, models can have identical log loss yet clearly different BAS (Appendix~\ref{sec:appendix_metrics_comparison}). Moreover, there is a key conceptual distinction: whereas log loss evaluates confidence as a probability forecast, BAS evaluates the expected utility of using confidence to make answer-or-abstain decisions under varying risk thresholds. This gives BAS a direct interpretation in terms of selective prediction behavior in abstention-aware settings.

\section{Experimental Setup}

We construct a comprehensive benchmark to evaluate the reliability of self-reported LLM confidence estimates. We evaluate a diverse set of models across multiple tasks, compare BAS with widely used metrics such as ECE and AURC, and study how confidence elicitation and calibration affect decision-level reliability. Our evaluation focuses on a black-box setting in which models are prompted to produce both an answer and a scalar confidence value.

\textbf{Benchmarks.}
We evaluate BAS across a diverse set of benchmarks designed to probe different failure modes of LLM confidence. Specifically, we consider AIME 2024/25~\citep{petrov2025proof} for complex multi-step mathematical reasoning, where incorrect answers often arise from reasoning errors; MedQA~\citep{singhal2025toward} for high-stakes medical question answering, where incorrect predictions can have significant consequences; and SimpleQA~\citep{wei2024measuring} for short-form factual QA, designed to test whether models abstain when knowledge is missing. These benchmarks differ in structure and difficulty, allowing us to evaluate how decision-useful confidence varies across domains.

\textbf{Models.}
We evaluate a diverse set of twelve models spanning different scales and training paradigms, including a reasoning-focused model (DeepSeek-R1~\citep{guo2025deepseek}), frontier models (DeepSeek-V3.2~\citep{liu2025deepseek}, Grok-3, Grok-3 Mini, GPT-4o~\citep{hurst2024gpt}, GPT-oss-120b~\citep{agarwal2025gpt}), smaller open-weight models (Llama 3.3 70B Instruct~\citep{dubey2024llama}, Phi-4~\citep{abdin2024phi}, GPT-oss~\citep{agarwal2025gpt}), as well as Mistral models in three sizes (Large, Medium, and Small)~\citep{liu2026ministral}. This diversity allows us to study how model scale and architecture affect confidence reliability.

\textbf{Confidence elicitation methods.} We evaluate four different methods to elicit confidence values $s \in [0,1)$ from LLMs. Our default method, \emph{direct elicitation}~\citep{lin2022teaching, tian2023just}, prompts the model for an answer and a confidence estimate simultaneously. We compare this to \emph{self-reflection}~\citep{kadavath2022language, mielke2022reducing, yin2023large}, which separates generation and confidence estimation into two distinct steps; \emph{top-$k$} elicitation~\citep{tian2023just}, where the model produces $k$ candidate answers with associated probabilities in a single response, from which we select the highest-probability answer and its confidence; and \emph{top-$k$ + self-reflection}, which combines both techniques. Full prompts for all evaluated elicitation methods are provided in Appendix~\ref{sec:prompt_appendix}.

\textbf{Post-hoc confidence calibration.}
In practice, LLM confidence estimates are often overconfident. We therefore study post-hoc calibration as a practical mechanism for improving the reliability of self-reported confidence. We apply \textit{isotonic regression}, a non-parametric monotonic mapping learned on a held-out validation set, to transform raw confidence values into calibrated estimates. Because BAS penalizes overconfident errors through the logarithmic term in Eq.~\ref{eq:bas_closed_form}, calibration directly affects decision-level utility by reducing the magnitude of these penalties while preserving the ranking of predictions. This setup allows us to evaluate how effective calibration is as a practical intervention, i.e., whether correcting confidence estimates on a validation set alone is sufficient to recover reliable answer-or-abstain behavior. Additional implementation details and analyses are provided in Appendix~\ref{sec:appendix_calibration_details}.

\section{Results}

We evaluate the reliability of self-reported LLM confidence across benchmarks, models, and elicitation methods, as well as the effect of post-hoc calibration. Our results reveal systematic variation in confidence reliability across tasks, trends in how it scales with model size and accuracy, and provide insights into how reliability may be improved in practice.

\subsection{Main Benchmark Results}
\label{sec:results_main}
Table~\ref{tab:results} reports accuracy, ECE, AURC, and BAS for all twelve models across SimpleQA, MedQA, and AIME (BAS results are also visualized in Figure~\ref{fig:overview}). First of all, we find that confidence reliability is strongly task-dependent. In structured multiple-choice settings such as MedQA, frontier and large models achieve positive BAS and near-zero ECE, indicating that their confidence estimates support reliable decisions. In contrast, in open-ended factual QA (SimpleQA), all evaluated models exhibit negative BAS and inflated ECE (ranging from 49.0\% to 84.2\%), reflecting a consistent tendency to produce overconfident incorrect answers. This suggests that even state-of-the-art LLMs struggle to express uncertainty appropriately when specific factual knowledge is absent.

\begin{table}[t]
\centering
\scriptsize
\setlength{\tabcolsep}{2.0pt}
\renewcommand{\arraystretch}{0.9}
\resizebox{\textwidth}{!}{%
\begin{tabular*}{\textwidth}{@{\extracolsep{\fill}}lcccccccccccc}
\toprule
\multicolumn{1}{c}{} & \multicolumn{4}{c}{\cellcolor{green!15}\textbf{SimpleQA}} & \multicolumn{4}{c}{\cellcolor{blue!15}\textbf{MedQA}} & \multicolumn{4}{c}{\cellcolor{red!15}\textbf{AIME}} \\
Model & \cellcolor{green!15}Acc$\uparrow$ & \cellcolor{green!15}BAS$\uparrow$ & \cellcolor{green!15}ECE$\downarrow$ & \cellcolor{green!15}AURC$\downarrow$ & \cellcolor{blue!15}Acc$\uparrow$ & \cellcolor{blue!15}BAS$\uparrow$ & \cellcolor{blue!15}ECE$\downarrow$ & \cellcolor{blue!15}AURC$\downarrow$ & \cellcolor{red!15}Acc$\uparrow$ & \cellcolor{red!15}BAS$\uparrow$ & \cellcolor{red!15}ECE$\downarrow$ & \cellcolor{red!15}AURC$\downarrow$\\
\midrule
\rowcolor{categorygray}
\multicolumn{13}{c}{\textbf{Frontier / Large Models (> 100B params)}} \\
\llmicon{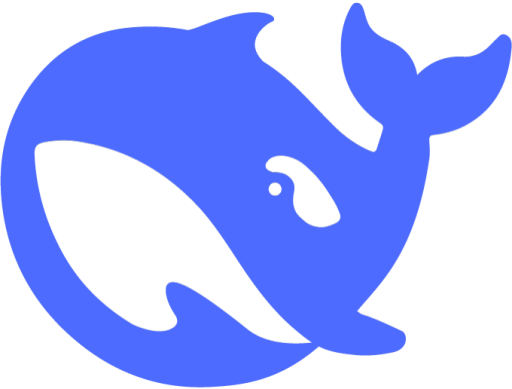} DS-R1 & 29.9$_{2.9}$ &\cellcolor{green!25} -0.48$_{0.06}$ & 52.7$_{2.7}$ & \cellcolor{green!25}0.60$_{0.04}$ & \cellcolor{green!25}92.1$_{1.5}$ & \cellcolor{green!25}0.76$_{0.03}$ & \cellcolor{green!25}2.6$_{1.2}$ & 0.03$_{0.01}$ & 66.7$_{11.7}$ & 0.45$_{0.20}$ & 17.6$_{8.1}$ & 0.07$_{0.06}$ \\
\llmicon{icon/deepseek_logo.png} DS-V3.2 & 17.6$_{2.3}$ & -0.63$_{0.06}$ & \cellcolor{green!25} 49.0$_{2.9}$ & 0.75$_{0.04}$ & 91.8$_{1.5}$ & \cellcolor{green!25}0.76$_{0.03}$ & 3.3$_{1.3}$ & \cellcolor{green!25}0.02$_{0.01}$ & 65.0$_{11.7}$ & 0.24$_{0.28}$ & 24.2$_{10.0}$ & 0.10$_{0.07}$ \\
\llmicon{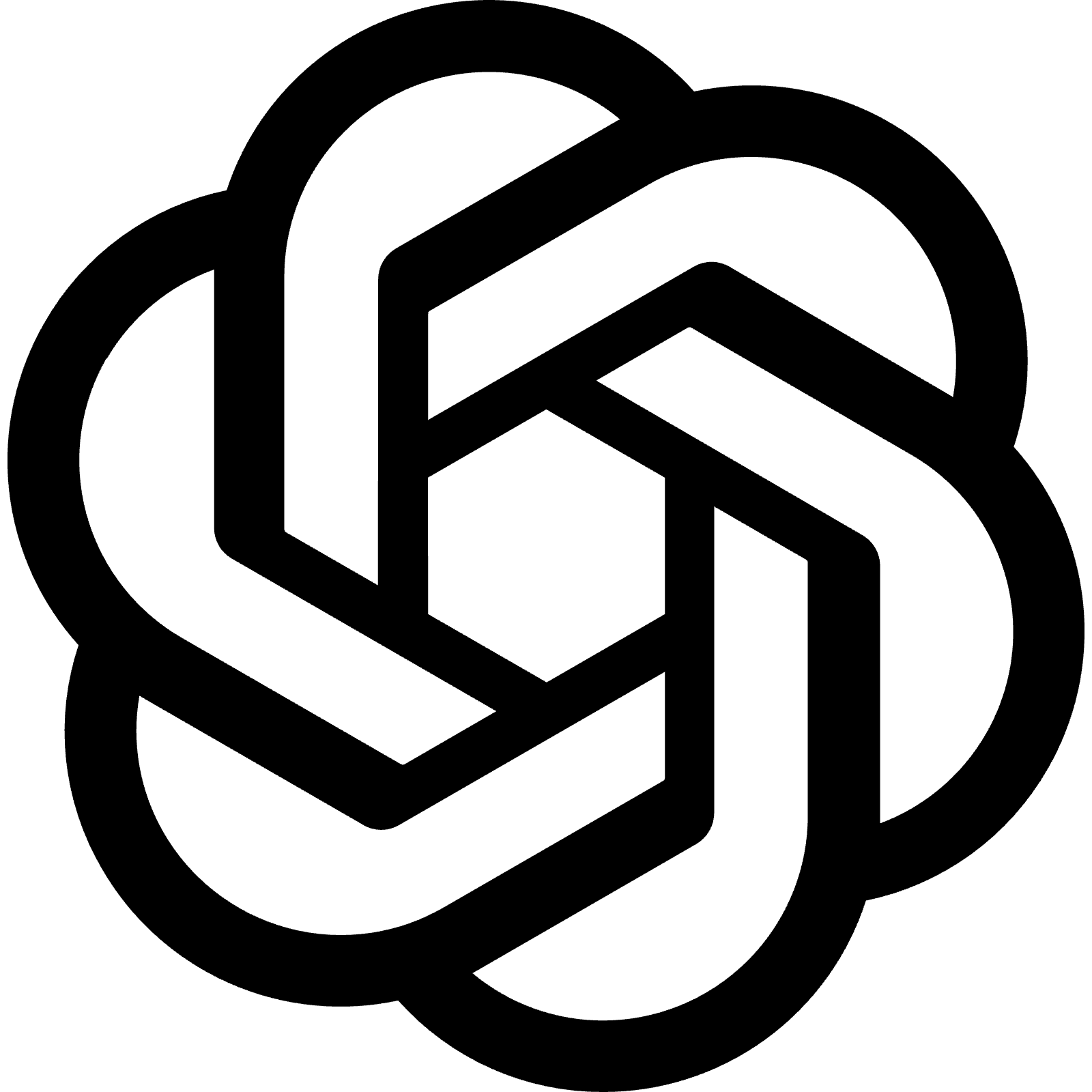} GPT-4o & 21.4$_{2.6}$ & -1.30$_{0.06}$ & 70.5$_{2.5}$ & 0.76$_{0.04}$ & 91.3$_{1.6}$ & 0.70$_{0.05}$ & 3.5$_{1.5}$ & 0.05$_{0.01}$ & 11.7$_{7.5}$ & -5.06$_{0.98}$ & 73.8$_{10.6}$ & 0.77$_{0.12}$ \\
\llmicon{icon/openai_logo.png} GPT-oss & 12.0$_{2.0}$ & -0.66$_{0.05}$ & 58.5$_{2.1}$ & 0.77$_{0.04}$ & 91.0$_{1.6}$ & 0.74$_{0.04}$ &\cellcolor{green!25} 2.6$_{1.2}$ & \cellcolor{green!25}0.02$_{0.01}$ & 75.0$_{10.8}$ &\cellcolor{green!25} 0.57$_{0.22}$ & 7.2$_{6.2}$ &\cellcolor{green!25} 0.05$_{0.04}$ \\
\llmicon{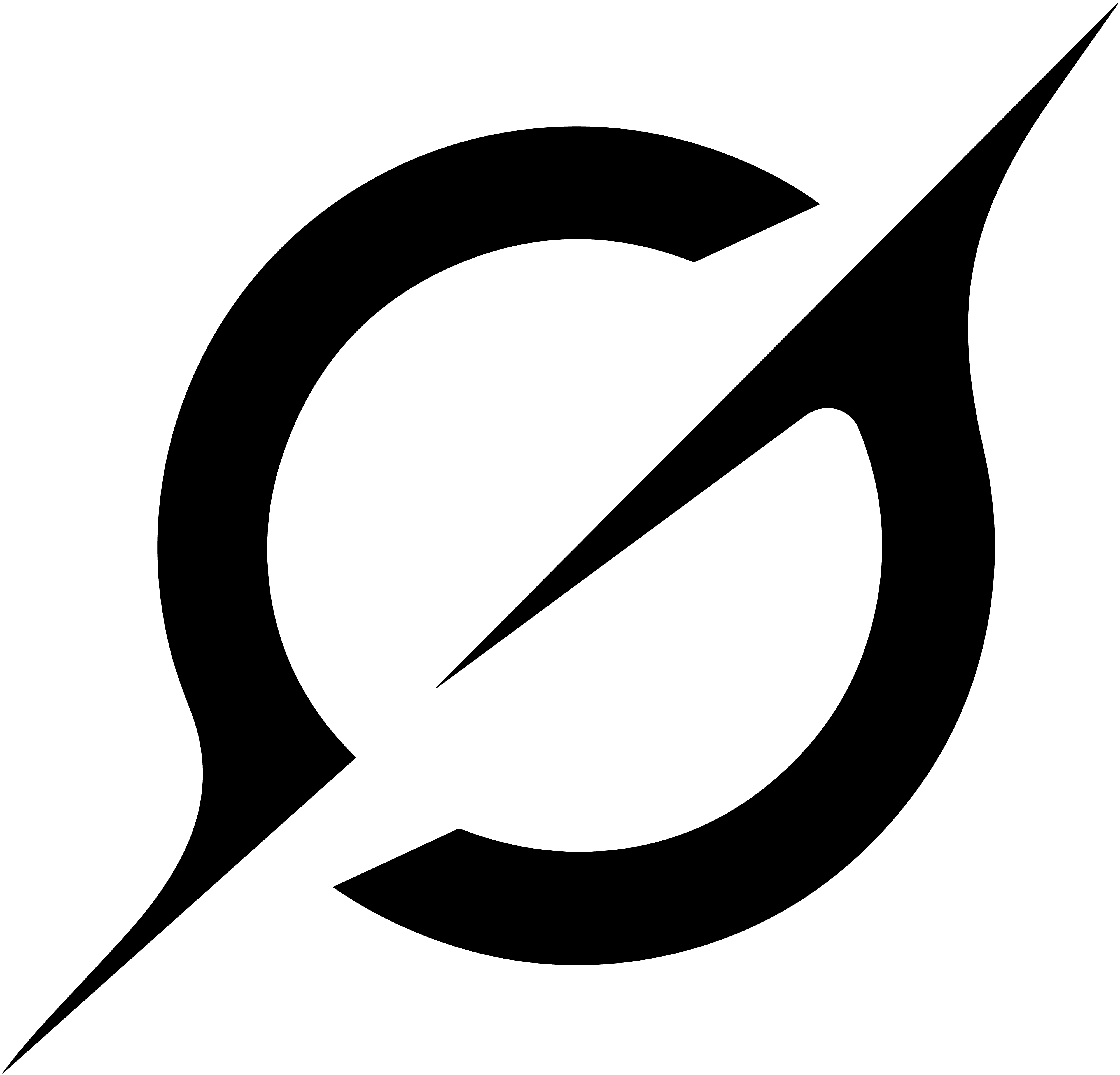} Grok-3 & \cellcolor{green!25} 32.9$_{2.9}$ & -0.90$_{0.08}$ & 60.1$_{2.8}$ & 0.65$_{0.04}$ & 91.0$_{1.6}$ & 0.73$_{0.03}$ &5.9$_{1.4}$ & 0.03$_{0.01}$ & 55.0$_{12.5}$ & 0.21$_{0.22}$ & 26.8$_{9.9}$ & 0.14$_{0.08}$ \\
\llmicon{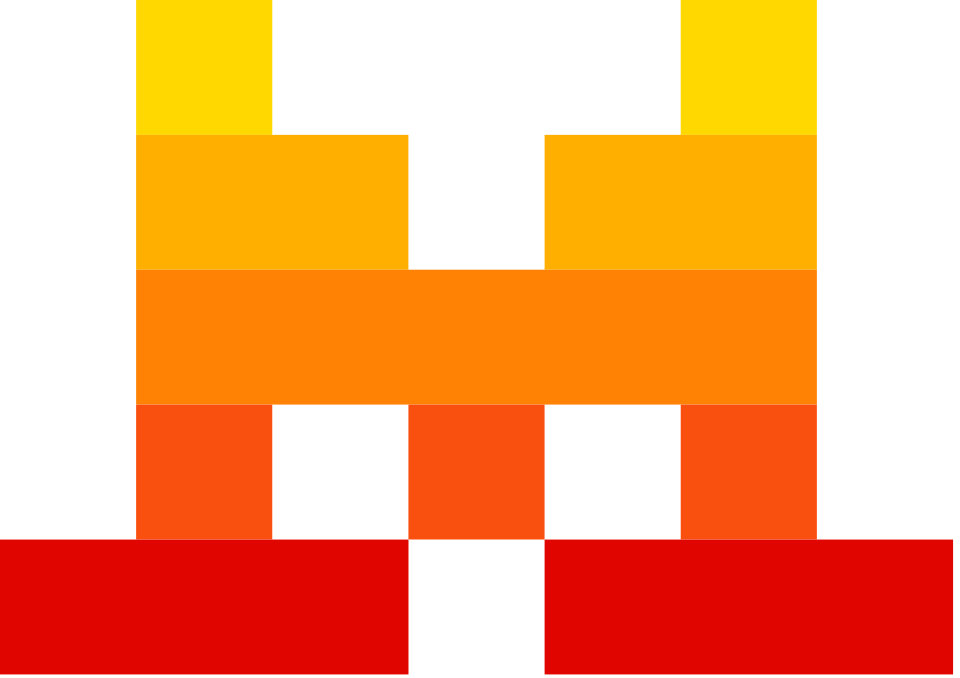} Mist(L) & 28.7$_{2.9}$ & -1.09$_{0.09}$ & 64.2$_{2.7}$ & 0.68$_{0.04}$ &\cellcolor{red!25} 75.3$_{2.3}$ & 0.51$_{0.05}$ & 11.1$_{1.9}$ & 0.07$_{0.01}$ & 43.3$_{11.7}$ & -0.31$_{0.35}$ & 44.0$_{11.3}$ & 0.26$_{0.13}$ \\
\addlinespace
\rowcolor{categorygray}
\multicolumn{13}{c}{\textbf{Mid-tier Models (10B to 70B params)}} \\
\llmicon{icon/grok.png} G3 Mini & 19.6$_{2.5}$ &-0.86$_{0.07}$ & 60.8$_{2.6}$ & 0.71$_{0.04}$ & 86.9$_{1.8}$ & 0.68$_{0.04}$ & 4.4$_{1.5}$ & 0.04$_{0.01}$ & \cellcolor{green!25}85.0$_{9.2}$ & 0.48$_{0.44}$ & \cellcolor{green!25} 5.9$_{6.3}$ &\cellcolor{green!25} 0.05$_{0.09}$ \\
\llmicon{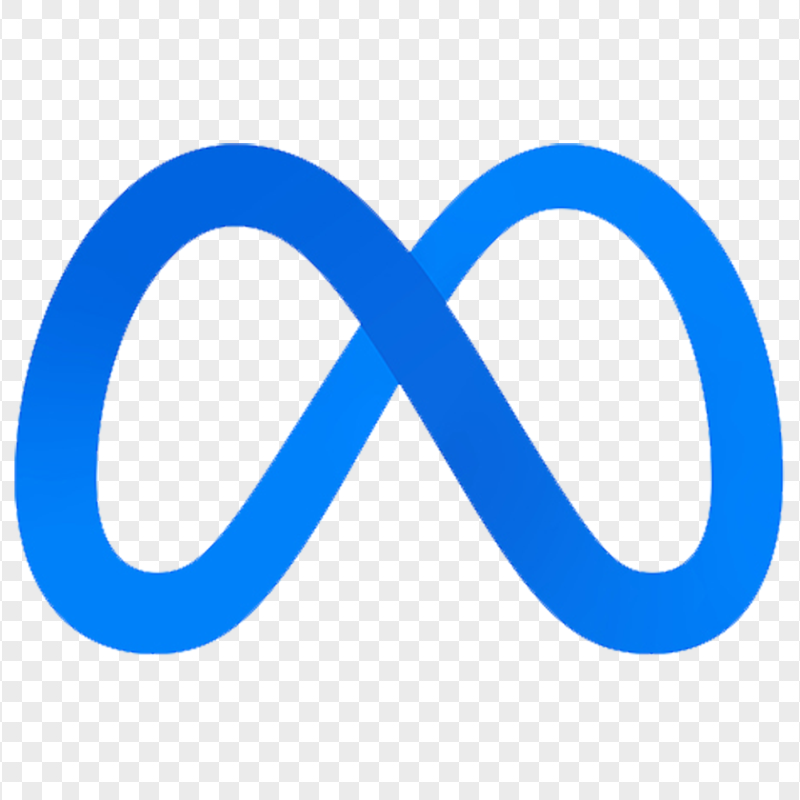} Llama 3.3 & 19.4$_{2.4}$ & -2.97$_{0.24}$ & 68.1$_{2.6}$ & 0.75$_{0.04}$ & 84.0$_{2.0}$ & 0.57$_{0.04}$ & 3.6$_{1.5}$ & 0.10$_{0.02}$ & 8.3$_{7.5}$ & -3.81$_{1.04}$ & 52.8$_{12.3}$ & 0.82$_{0.12}$ \\
\llmicon{icon/mistral.png} Mist(M) & 20.2$_{2.5}$ & -1.26$_{0.12}$ & 68.8$_{2.4}$ & 0.74$_{0.04}$ & 84.8$_{2.0}$ & 0.57$_{0.05}$ & 6.4$_{1.9}$ & 0.08$_{0.02}$ & 30.0$_{11.7}$ & -1.42$_{0.75}$ & 54.8$_{10.8}$ & 0.50$_{0.17}$ \\
\addlinespace
\rowcolor{categorygray}
\multicolumn{13}{c}{\textbf{Small / Efficient Models (<10B params or highly optimized)}} \\
\llmicon{icon/openai_logo.png} 4o-mini & 8.9$_{1.8}$ & \cellcolor{red!25}-3.86$_{0.21}$ & \cellcolor{red!25}84.2$_{1.8}$ & 0.89$_{0.03}$ & 79.8$_{2.2}$ & 0.48$_{0.05}$ & 10.6$_{2.1}$ & 0.10$_{0.02}$ & 11.7$_{7.5}$ &-5.01$_{0.95}$ & 81.4$_{9.3}$ & 0.78$_{0.13}$ \\
\llmicon{icon/mistral.png} Mist(S) & 10.9$_{2.0}$ & -1.63$_{0.07}$ & 82.6$_{1.9}$ & 0.89$_{0.03}$ & 80.0$_{2.2}$ & 0.47$_{0.05}$ & 11.4$_{2.1}$ & \cellcolor{red!25}0.11$_{0.02}$ & \cellcolor{red!25}5.0$_{5.8}$ & -1.68$_{0.56}$ & 62.5$_{11.4}$ & \cellcolor{red!25}0.88$_{0.11}$ \\
\llmicon{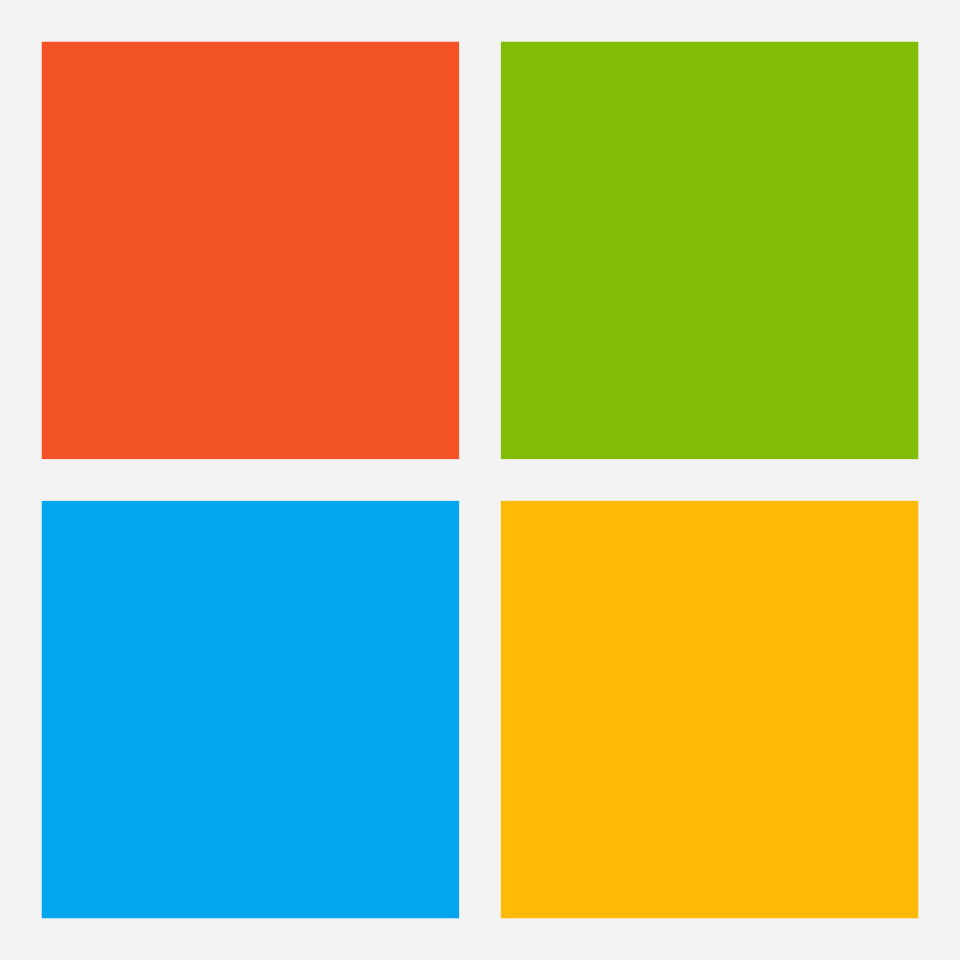} Phi-4 &\cellcolor{red!25} 8.5$_{1.7}$ & -3.63$_{0.22}$ & 84.1$_{1.8}$ & \cellcolor{red!25}0.91$_{0.03}$ & 80.8$_{2.2}$ &\cellcolor{red!25} 0.39$_{0.07}$ & \cellcolor{red!25}14.2$_{2.1}$ & \cellcolor{red!25}0.11$_{0.02}$ & 15.0$_{9.2}$ &\cellcolor{red!25} -6.52$_{0.87}$ & \cellcolor{red!25}84.7$_{9.1}$ & 0.80$_{0.12}$ \\
\addlinespace
\bottomrule
\end{tabular*}
}%
\vspace{-1.5mm}
\caption{Evaluation of models across SimpleQA, MedQA, and AIME. Results highlight that confidence reliability is strongly task-dependent and does not always align with accuracy, with even frontier models exhibiting severe overconfidence in challenging settings. }\vspace{-2.0mm}
\label{tab:results}
\end{table}

Differences between models are most apparent in reasoning-heavy tasks such as AIME. For example, GPT-4o achieves 11.7\% accuracy with a strongly negative BAS of -5.06, whereas GPT-oss achieves 75.0\% accuracy and a positive BAS of 0.57. While GPT-4o outperforms smaller models such as Mistral (S) in terms of accuracy, its substantially lower BAS indicates much more severe overconfidence in its incorrect predictions. Similarly, on SimpleQA, GPT-oss achieves lower accuracy but substantially better BAS than Llama 3.3, further highlighting that model capability and confidence reliability do not always align.

\textbf{Takeaway.}
Confidence reliability is both task- and model-dependent: models can accurately estimate their confidence on some tasks while exhibiting severe overconfidence on others, and even frontier models may fail to provide reliable confidence in challenging settings.

\subsection{Scaling and Performance Trends}

Given the task- and model-dependent variation observed in Section~\ref{sec:results_main}, we next examine how confidence reliability relates to model scale and predictive performance. Figure~\ref{fig:scale_accuracy_metrics} (top) shows accuracy, BAS, ECE, and AURC as a function of model size, for a subset of models with known parameter counts. Across these models, we observe a consistent trend in which larger models tend to achieve \emph{both} higher accuracy and more reliable confidence estimates, although substantial variation remains across models. Figure~\ref{fig:scale_accuracy_metrics} (bottom) shows a similar relationship when comparing predictive performance and reliability directly, for all twelve evaluated LLMs. Models with higher accuracy tend to exhibit higher BAS and lower ECE and AURC, suggesting that better models are not only more accurate, but also tend to provide more reliable confidence estimates. While these relationships are not uniform, and several models deviate from the overall trends in Figure~\ref{fig:scale_accuracy_metrics}, this suggests that improvements in model capability are often accompanied by improvements in confidence estimation.

\textbf{Takeaway.}
Confidence reliability generally improves both with model scale and predictive performance: larger and more accurate models tend to also provide more reliable, decision-useful confidence estimates.

\begin{figure}[t]
\centering
    \begin{subfigure}[t]{1.0\textwidth}
        \centering%
        \includegraphics[clip, trim=0.0cm 1.5cm 0.0cm 0.0cm, width=1.0\linewidth]{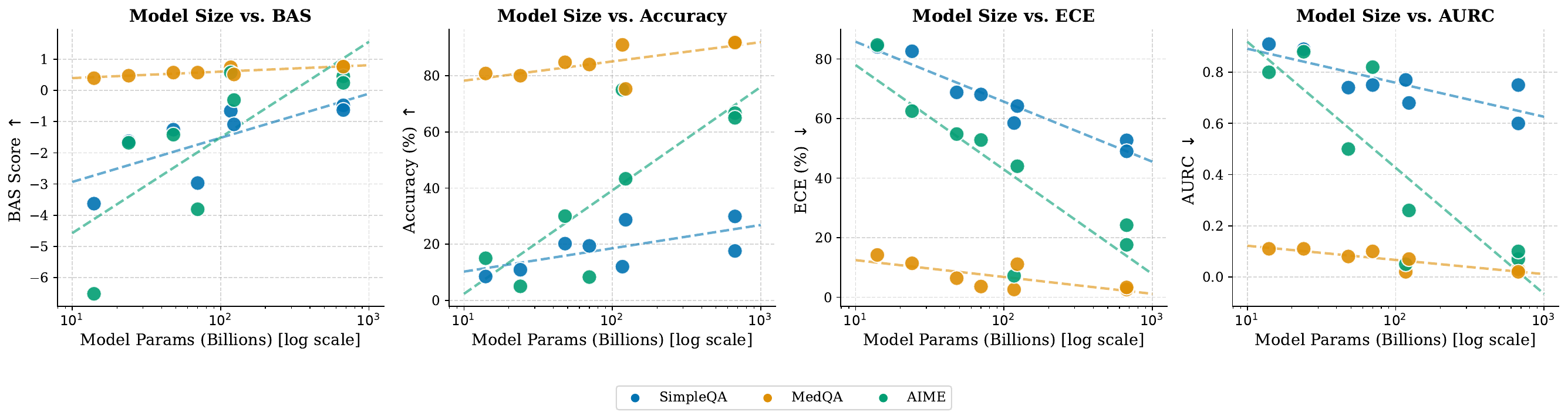}
    \end{subfigure}
    \begin{subfigure}[t]{1.0\textwidth}
        \centering%
        \includegraphics[width=0.70\linewidth]{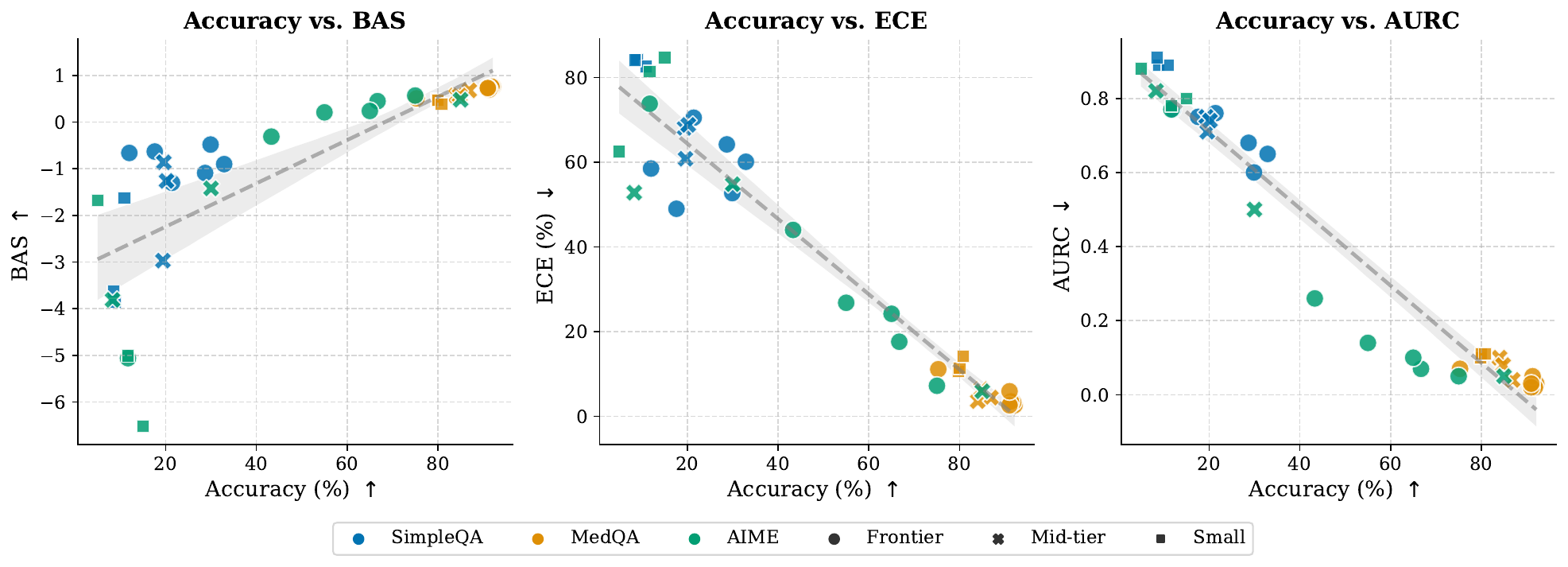}
    \end{subfigure}\vspace{-1.5mm}
  \caption{Relationship between model scale, predictive performance, and confidence reliability. \textbf{Top:} Model size vs. accuracy and reliability metrics. Larger models tend to achieve higher accuracy \textit{and} improved confidence reliability (higher BAS, lower ECE and AURC), although substantial variation remains across models. \textbf{Bottom:} Accuracy vs. reliability metrics. Models with higher accuracy tend to also exhibit more reliable confidence estimates.}\vspace{-2.0mm}
  \label{fig:scale_accuracy_metrics}
\end{figure}

\subsection{Confidence Elicitation Methods}
\label{section:results_conf_elicitation}

Because black-box LLMs do not expose native probability distributions, the method used to elicit confidence is a key factor in determining reliability. Table~\ref{tab:conf_elicitation_results} reports BAS across four confidence elicitation methods on SimpleQA, for a subset of five representative models, while Figure~\ref{fig:conf_elicitation_results} provides results also for accuracy, ECE, and AURC.

We find that the elicitation method has a large impact on confidence reliability. Under the default \emph{direct elicitation} method, all models exhibit clear overconfidence, resulting in consistently negative BAS. Prompting models to instead consider multiple candidate answers (\emph{top-$k$}, with $k = 3$) consistently improves reliability, yielding the highest BAS across all models. For example, Llama-3.3-70B improves from -2.97 to -0.25, indicating a substantial reduction in overconfident errors. Figure~\ref{fig:conf_elicitation_results} shows that accuracy and AURC remain largely unchanged across methods, whereas ECE consistently improves when moving away from direct elicitation. While \emph{self-reflection} improves ECE compared to direct elicitation, it yields mixed results in terms of BAS. In particular, combining it with top-$k$ (\emph{top-$k$ + self-reflection}) consistently degrades BAS compared to top-$k$ alone. From a practical perspective, top-$k$ is particularly appealing: it requires only a single forward pass, unlike self-reflection methods which require a two-step prompting procedure.

\textbf{Takeaway.}
Confidence elicitation strongly affects reliability: simple strategies such as top-$k$ can substantially reduce overconfidence and improve calibration, without increasing inference cost or significantly affecting model accuracy.

\subsection{Effect of Post-hoc Calibration}
We next study post-hoc calibration based on a validation set as a practical mechanism for improving confidence reliability. Table~\ref{tab:calibration} reports BAS and ECE before and after applying isotonic regression on SimpleQA, for a subset of five representative models.

As previously observed, raw confidence estimates are overconfident, leading to high ECE and substantially negative BAS across models. After calibration, ECE decreases dramatically, typically from 60-90\% to low single digits, indicating that reported confidence becomes better aligned with empirical accuracy. BAS also improves substantially, often shifting from strongly negative values to near-zero. For example, Llama-3.3-70B improves from -2.97 to 0.02, whereas GPT-4o-mini improves from -3.86 to 0.001. We further analyze sensitivity to the validation set size for GPT-oss-120B in Table~\ref{tab:calibration_ablation}, finding that most gains already emerge with relatively small validation sets. These results suggest that post-hoc calibration can improve confidence reliability with limited additional data, although its effectiveness in practice will depend on how well the validation set reflects the deployment setting.

\textbf{Takeaway.}
Post-hoc calibration is a simple option that can substantially improve confidence reliability by reducing overconfidence, and may provide meaningful gains in practice even with relatively small validation sets.

\begin{table}[t]
\centering
\small
\renewcommand{\arraystretch}{1.1}
\resizebox{0.875\textwidth}{!}{%
\begin{tabular}{lcccc}
\toprule
\textbf{Model} & \textbf{Direct Elicitation} & \textbf{Self-Reflection} & \textbf{Top-$k$} & \textbf{Top-$k$ + Self-Reflection} \\
\midrule
\llmicon{icon/openai_logo.png} GPT-4o-mini          & $-3.86 \pm 0.21$ & $-0.94 \pm 0.09$ & $\mathbf{-0.42 \pm 0.02}$ & $-0.91 \pm 0.14$ \\
\llmicon{icon/openai_logo.png} GPT-oss-120B         & $-0.66 \pm 0.04$ & $-0.61 \pm 0.05$ & $\mathbf{-0.22 \pm 0.02}$ & $-0.31 \pm 0.03$ \\
\llmicon{icon/grok.png} Grok-3-Mini          & $-0.86 \pm 0.07$ & $-0.92 \pm 0.08$ & $\mathbf{-0.50 \pm 0.05}$ & $-0.58 \pm 0.07$ \\
\llmicon{icon/llama.png} Llama-3.3-70B        & $-2.97 \pm 0.23$ & $-2.05 \pm 0.21$ & $\mathbf{-0.25 \pm 0.06}$ & $-0.54 \pm 0.11$ \\
\llmicon{icon/microsoft.png} Phi-4                & $-3.63 \pm 0.21$ & $-3.60 \pm 0.22$ & $\mathbf{-0.61 \pm 0.08}$ & $-2.23 \pm 0.20$ \\
\bottomrule
\end{tabular}%
}
\vspace{-1.5mm}
\caption{Comparison of BAS across four confidence elicitation methods on SimpleQA. Top-$k$ consistently yields the highest BAS across a subset of five representative models.}\vspace{-2.0mm}
\label{tab:conf_elicitation_results}
\end{table}

\begin{table}[t]
\centering
\small
\resizebox{0.70\textwidth}{!}{%
\begin{tabular}{l cc cc}
\toprule
& \multicolumn{2}{c}{\textbf{BAS $\uparrow$}} 
& \multicolumn{2}{c}{\textbf{ECE (\%) $\downarrow$}} \\
\cmidrule(lr){2-3} \cmidrule(lr){4-5}
\textbf{Model} 
& Before & After 
& Before & After \\
\midrule
\llmicon{icon/openai_logo.png} GPT-4o-mini   & $-3.86 \pm 0.21$ & $\mathbf{\phantom{-}0.001 \pm 0.001}$ & $88.6 \pm 1.4$ & $\mathbf{3.6 \pm 1.3}$ \\
\llmicon{icon/openai_logo.png} GPT-oss-120B  & $-0.66 \pm 0.04$ & $\mathbf{-0.000 \pm 0.006}$           & $64.2 \pm 1.8$ & $\mathbf{6.0 \pm 1.4}$ \\
\llmicon{icon/grok.png} Grok-3-Mini   & $-0.88 \pm 0.07$ & $\mathbf{\phantom{-}0.026 \pm 0.008}$ & $60.7 \pm 2.7$ & $\mathbf{1.7 \pm 1.7}$ \\
\llmicon{icon/llama.png} Llama-3.3-70B & $-2.97 \pm 0.23$ & $\mathbf{\phantom{-}0.024 \pm 0.007}$ & $68.1 \pm 2.6$ & $\mathbf{1.5 \pm 1.7}$ \\
\llmicon{icon/microsoft.png} Phi-4         & $-3.63 \pm 0.21$ & $\mathbf{\phantom{-}0.004 \pm 0.001}$ & $84.1 \pm 1.8$ & $\mathbf{1.8 \pm 1.4}$ \\
\bottomrule
\end{tabular}
}
\vspace{-1.5mm}
\caption{Effect of post-hoc calibration via isotonic regression on SimpleQA. Calibration substantially reduces ECE and improves BAS across a subset of five representative models.}\vspace{-2.0mm}
\label{tab:calibration}
\end{table}

\subsection{Comparison of BAS with Existing Metrics}
\label{section:results_metrics_comp}

We first analyze how BAS differs from the standard reliability metrics ECE and AURC. Figure~\ref{fig:bas_vs_metrics} shows the relationship between BAS and these metrics across all model-dataset pairs. While BAS is generally correlated with both ECE and AURC, we observe notable deviations. In particular, similar ECE can be achieved with substantially different BAS.

For example, on AIME, Llama~3.3 and Mistral~(M) achieve similar ECE (52.8\% and 54.8\%, Table~\ref{tab:results}), but differ substantially in BAS (-3.81 vs.\ -1.42). In this case, AURC also differs (0.82 vs.\ 0.50), indicating that both ranking and confidence magnitude contribute to the discrepancy in decision-level reliability. However, similar divergences appear even when AURC is nearly identical. On SimpleQA, Llama~3.3 and Mistral~(M) achieve nearly identical ECE (68.1\% and 68.8\%) \emph{and} AURC (0.75 and 0.74), yet exhibit substantially different BAS (-2.97 vs.\ -1.26). These differences can be explained by the distribution of confidence values, as shown in Figure~\ref{fig:confidence_histograms}. While both Llama~3.3 and Mistral~(M) assign high confidence to many predictions, Llama~3.3 exhibits a stronger concentration of incorrect predictions at very high confidence (near $s \approx 1$). These highly overconfident errors contribute disproportionately to BAS due to its logarithmic penalty, while having only a limited effect on ECE. As a result, models with similar ECE (and even similar AURC) can exhibit substantially different BAS when their confidence distributions differ in the tail behavior.

Next, we compare BAS with the proper scoring rule log loss. While BAS structurally differs from log loss by imposing an asymmetric penalty that strongly prioritizes avoiding overconfident errors, Figure~\ref{fig:bas_vs_logloss} shows that they are highly correlated across model-dataset pairs in practice. This can also be explained by the confidence distribution (Figure~\ref{fig:confidence_histograms}): since modern LLMs are much more prone to overconfidence than underconfidence, both BAS and log loss are dominated by the logarithmic penalty $\ln(1 - s)$ on highly confident errors in our evaluation. Nevertheless, the structural difference remains, enabling BAS to distinguish between models that trade off overconfidence with underconfidence (Appendix~\ref{sec:appendix_metrics_comparison_logloss}). This asymmetry is well motivated from a decision-theoretic perspective, where overly conservative but safe behavior induced by underconfidence is often preferable to high-confidence errors.

\textbf{Takeaway.}
ECE and AURC capture complementary but incomplete aspects of model confidence, whereas BAS provides additional insight by identifying highly overconfident errors that dominate decision-level risk, and by distinguishing models that differ in how they trade off overconfidence and underconfidence.

\section{Related Works}

\textbf{Selective prediction and abstention.}
Selective prediction, or a model's ability to reject uncertain inputs, studies the trade-off between accuracy and coverage \citep{chow2003optimum, el2010foundations, geifman2017selective}. Recent work extends this to LLMs and QA tasks. Abstain-QA and the Answerable Unanswerable Confusion Matrix (AUCM) provide black-box evaluation methods suited for proprietary models \citep{madhusudhan2024llms}. Large-scale benchmarks like AbstentionBench test unanswerability from underspecified inputs, false premises, stale data, and subjective queries, showing that reasoning fine-tuning can worsen abstention \citep{kirichenko2025abstentionbench}. System-level studies also integrate abstention into LLM pipelines \citep{zellinger2025cost}. 

\textbf{Uncertainty quantification and calibration.} Uncertainty quantification is a popular field that focuses on estimating and calibrating model uncertainty. Such methods include token- or sequence-level likelihoods, sampling-based dispersion metrics, and verbalized/self-reflected confidence scores \citep{farquhar2024detecting, phillips2025geometric, phillips2026semantic, phillips2026entropyinsufficientsafeselective,li2025semanticvolumequantifyingdetecting, taubenfeld2025confidence, kadavath2022language}. Other research analyzes how alignment processes (instruction tuning / RLHF) can degrade logit-based calibration in the multiple-choice setting \citep{he2023investigating}. 

\textbf{Alignment, truthfulness, and post-training objectives.}
A number of recent works aim to reduce hallucination by changing training objectives. TruthRL uses an RL objective with a ternary reward to incentivize truthfulness \citep{wei2025truthrl}, though such RL methods can be unstable on smaller models \citep{wu2025mitigating}. Alternative approaches utilize Supervised Fine-Tuning (SFT) on carefully curated `refusal datasets' or rejection tokens to teach models when to abstain \citep{huang2025alleviating, zhang2025reasoning, jain2024refusal, xu2024editing}. 

\textbf{Summary and distinction.}
Prior work on LLM uncertainty has focused on estimating or calibrating confidence, as well as on metrics such as ECE and AURC. However, these widely used calibration- and ranking-based metrics do not capture the decision-level impact of overconfident errors. In contrast, BAS provides a unified, black-box, decision-theoretic metric that evaluates answer-or-abstain behavior under an explicit cost model. By asymmetrically penalizing highly confident errors, BAS captures aspects of confidence reliability that are not reflected in calibration, ranking, or symmetric scoring rules alone.

\section{Conclusion}

We introduce BAS, a decision-theoretic framework for evaluating LLM confidence reliability by rewarding calibrated abstention and penalizing overconfident errors. Our theoretical analysis shows that BAS defines a proper scoring objective, uniquely maximized when a model's expressed confidence matches its true probability of correctness. Empirically, we find that confidence reliability varies substantially across tasks, and while larger and more accurate models tend to also achieve higher BAS, even frontier models remain prone to severe overconfidence. We further show that models with similar calibration (ECE) or ranking (AURC) can exhibit highly different decision-level reliability under BAS, driven by rare but highly overconfident errors. While BAS is related to proper scoring rules such as log loss, it differs structurally by imposing an asymmetric penalty that strongly prioritizes avoiding overconfident errors. We also find that confidence reliability can be improved through simple interventions, such as top-$k$ elicitation and post-hoc calibration, often with relatively low additional cost. Overall, BAS provides a principled and practical framework for evaluating LLM confidence reliability and complements standard metrics in risk-sensitive settings.

\section*{Acknowledgments}

\textbf{Sean Wu} was supported by the Rhodes Scholarship. \textbf{Edward Phillips} was funded by an NIHR Research Studentship. \textbf{David A. Clifton} was funded by an NIHR Research Professorship; a Royal Academy of Engineering Research Chair; and the InnoHK Hong Kong Centre for Cerebro-cardiovascular Engineering (COCHE); and was supported by the National Institute for Health Research (NIHR) Oxford Biomedical Research Centre (BRC) and the Pandemic Sciences Institute at the University of Oxford.

\bibliography{colm2026_conference}
\bibliographystyle{colm2026_conference}

\renewcommand{\thefigure}{A\arabic{figure}}
\setcounter{figure}{0}
\renewcommand{\thetable}{A\arabic{table}}
\setcounter{table}{0}
\renewcommand{\theequation}{A\arabic{equation}}
\setcounter{equation}{0}

\newpage
\appendix

\section{BAS Theoretical Analysis}
\label{sec:appendix_bas_theoretical}

We now establish the theoretical soundness of the BAS framework. The following theorem (which is the same as Theorem~\ref{thm:bas_optimal}, repeated here for convenience) demonstrates that the BAS objective is decision-theoretically sound, uniquely maximizing expected utility when the model's expressed uncertainty aligns with its true probability of correctness.

\begin{theorem}[Optimality of BAS Utility]
\label{thm:bas_optimal_appendix}
Let $s \in [0, 1)$ be the reported model confidence and $p \in [0, 1]$ the true probability of correctness. The expected BAS utility (Eq.~\ref{eq:expected_bas}) is uniquely maximized when $s = p$ for all $p < 1$. For $p=1$, the expected utility is strictly increasing in $s$ and achieves its supremum as $s \to 1$.
\end{theorem}

\begin{proof}
Let $Z \in \{0, 1\}$ be a random variable indicating the correctness of the model's response, with $P(Z=1) = p$. The BAS framework defines a deterministic decision policy where the model answers if the risk threshold $t \le s$ and abstains otherwise.

The BAS score aggregates the selective utility $S_t$ (Eq. \ref{eq:utility}) across all risk thresholds $t \in [0, 1)$, assuming a uniform prior over risk tolerance. For a fixed threshold $t \le s$, the expected utility is:
\begin{equation*}
\mathbb{E}[S_t] = p - (1-p)\frac{t}{1-t}.
\end{equation*}
For $t > s$, the model abstains, yielding zero utility. Thus, the total expected BAS utility $\mathbb{E}[U(s)] = \int_0^1 \mathbb{E}[S_t]\, dt$ is obtained by integrating over the active region $[0, s]$:
\begin{equation*}
\mathbb{E}[U(s)] = \int_0^s \mathbb{E}[S_t]\, dt = \int_{0}^{s} \left( p - (1-p)\frac{t}{1-t} \right) dt = p \int_{0}^{s} 1 \, dt + (1-p) \int_{0}^{s} -\frac{t}{1-t} \, dt.
\end{equation*}
Evaluating the integrals:
\begin{align*}
\mathbb{E}[U(s)] &= ps + (1-p)\left[ s + \ln(1-s) \right] \\
&= s + (1-p)\ln(1-s).
\end{align*}
We analyze the maximization of this objective in two cases:

\textbf{Case 1: $p < 1$.} 
The first derivative with respect to $s$ is:
\begin{equation*}
\frac{\partial \mathbb{E}[U(s)]}{\partial s} = 1 - \frac{1-p}{1-s}.
\end{equation*}
Setting this to zero yields $1-s = 1-p$, implying $s = p$. The second derivative is:
\begin{equation*}
\frac{\partial^2 \mathbb{E}[U(s)]}{\partial s^2} = -\frac{1-p}{(1-s)^2}.
\end{equation*}
Since $p < 1$, the second derivative is strictly negative for all $s \in [0, 1)$. Thus, $\mathbb{E}[U(s)]$ is strictly concave and uniquely maximized at $s = p$.

\textbf{Case 2: $p = 1$.}
Substituting $p=1$ into the expected utility eliminates the penalty term, simplifying the objective to $\mathbb{E}[U(s)] = s$. This function is strictly increasing on $[0, 1)$. Consequently, the expected utility is maximized (as a supremum) as $s \to 1$, consistent with truthful reporting of certainty.
\end{proof}

\section{Weighted BAS for Safety-Critical Settings}
\label{sec:weighted_bas_consolidated}
The standard $\mathrm{BAS}$ metric assumes a uniform prior over the risk tolerance $t \in [0,1)$. However, real-world deployments often induce specific risk profiles (e.g., safety-critical settings place more weight on high thresholds). We generalize the metric by introducing a risk-prior $w(t)$ satisfying $\int_0^1 w(t) dt = 1$. The expected \emph{weighted} BAS utility is then defined as:
\begin{equation}
\label{eq:weighted_bas_expected}
\mathbb{E}[U_w(s)] := \int_{0}^{1} \mathbb{E}[S_t(Z, \pi_s(t))] \, w(t) \, dt,
\end{equation}
which is a direct generalization of Eq.~\ref{eq:expected_bas}. For evaluation, we compute the \emph{realized} per-example utility $U_w(s, Z)$ via integration over the answering region $t \in [0, s]$:
\begin{equation}
\label{eq:weighted_bas_realized}
U_w(s, Z) \;=\; \int_{0}^{s} \left[ Z \cdot 1 + (1-Z)\left(-\frac{t}{1-t}\right) \right] w(t) \, dt.
\end{equation}
For general weights $w(t)$ where no closed form solution exists, this contribution is computed via numerical quadrature.

Crucially, the property of truthfulness holds even under weighted aggregation:
\begin{theorem}[Optimality under Weighted Risk]
\label{thm:weighted_bas}
Let $s \in [0, 1)$ be the reported model confidence and $p \in [0, 1]$ the true probability of correctness. Let $w: [0, 1) \to \mathbb{R}_{\ge 0}$ be any integrable weighting function with $\int_0^1 w(t)\,dt = 1$ such that $w(t) > 0$ almost everywhere. Then the expected weighted BAS utility (Eq.~\ref{eq:weighted_bas_expected}) is uniquely maximized when $s = p$ for all $p < 1$. For $p=1$, the expected weighted utility is strictly increasing in $s$ and attains its supremum as $s\to 1$.
\end{theorem}
\begin{proof}
Under the policy ``answer if and only if $t\le s$'', the expected selective utility is
\[
\mathbb{E}[S_t] =
\begin{cases}
p - (1-p)\frac{t}{1-t}, & t\le s,\\
0, & t>s.
\end{cases}
\]
Hence
\begin{equation}
\label{eq:Uw_integral}
\mathbb{E}[U_w(s)]
=
\int_0^s \left(p - (1-p)\frac{t}{1-t}\right)\, w(t)\,dt.
\end{equation}
Assuming $w$ is integrable, differentiation under the integral sign yields for $s\in(0,1)$:
\begin{align*}
\frac{d}{ds}\mathbb{E}[U_w(s)]
&=
\left(p - (1-p)\frac{s}{1-s}\right)\, w(s) \\
&=
\left( \frac{p(1-s) - s(1-p)}{1-s} \right) w(s) \\
&=
\frac{p-s}{1-s}\, w(s).
\end{align*}
Since $w(s)>0$ almost everywhere on $(0,1)$, the condition $\frac{d}{ds}\mathbb{E}[U_w(s)] = 0$ holds if and only if $s=p$. Moreover, for $p<1$, the derivative is positive for $s<p$ and negative for $s>p$, implying $s=p$ is the unique maximizer.

For $p=1$, the term in the integral of Eq.~\ref{eq:Uw_integral} becomes $1$, so $\mathbb{E}[U_w(s)]=\int_0^s w(t)\,dt$, which is strictly increasing in $s$ (since $w$ is positive) and achieves its supremum as $s\to 1$.
\end{proof}

To ensure standardized benchmarking and leaderboard compatibility, we establish the uniform prior ($w(t) = 1$) as the default $\mathrm{BAS}$ metric for all general-purpose LLM evaluation. The weighted variants are treated as domain-specific extensions for safety-critical settings.

\subsection{Weighted BAS Results}
We show in Table~\ref{tab:weighted_bas} the results of varying the risk-prior $w(t)$ and how it changes the relative performance across LLMs.

\begin{table}[h]
\centering
\small
\begin{tabular}{@{}llccc@{}}
\toprule
\textbf{Risk Profile} & \textbf{Weighting $w(t)$} & \textbf{SimpleQA} & \textbf{AIME} & \textbf{MedQA} \\
\midrule
\rowcolor{categorygray}
\multicolumn{5}{c}{\raisebox{-0.2\height}{\includegraphics[height=1.0em]{icon/deepseek_logo.png}} \quad \textbf{DeepSeek-R1} \vphantom{fg}} \\
General Purpose & Uniform: $w(t) = 1$ &  \textbf{-0.4847}  &  \textbf{0.4492} &  \textbf{0.7640} \\
Risk-Aware      & Linear: $w(t) = 2t$ & -0.7904  & 0.3706 & 0.6688 \\
Safety-Critical & Quadratic: $w(t) = 3t^2$ &  -0.9291 & 0.3299 &  0.5989 \\
\midrule
\rowcolor{categorygray}
\multicolumn{5}{c}{\raisebox{-0.2\height}{\includegraphics[height=1.0em]{icon/openai_logo.png}} \quad \textbf{GPT-4o} \vphantom{fg}} \\
General Purpose & Uniform: $w(t) = 1$ & \textbf{-1.1125} &  \textbf{-5.0629} &  \textbf{0.7030} \\
Risk-Aware      & Linear: $w(t) = 2t$ & -1.7744 &-9.5149 & 0.5680 \\
Safety-Critical & Quadratic: $w(t) = 3t^2$ & -2.1562 & -13.6129 & 0.4660 \\
\midrule
\rowcolor{categorygray}
\multicolumn{5}{c}{\raisebox{-0.2\height}{\includegraphics[height=1.0em]{icon/grok.png}} \quad \textbf{Grok-3} \vphantom{fg}} \\
General Purpose & Uniform: $w(t) = 1$ & \textbf{-0.8986}& \textbf{0.2101} & \textbf{0.7319} \\
Risk-Aware      & Linear: $w(t) = 2t$ & -1.5468 & 0.1017 & 0.6286  \\
Safety-Critical & Quadratic: $w(t) = 3t^2$ & -1.9446 & 0.0611 &  0.5556 \\
\bottomrule
\end{tabular}
\caption{Comparison of BAS for frontier models across tasks when varying the risk-prior $w(t)$. Higher scores indicate better alignment with expressed uncertainty. The standard $\mathrm{BAS}$ (Uniform: $w(t) = 1$) serves as the baseline, while Linear and Quadratic weights simulate increasingly safety-critical environments.}
\label{tab:weighted_bas}
\end{table}

\section{Relation to Existing Metrics: Additional Details}
\label{sec:appendix_metrics_comparison}

\subsection{Illustrative Examples: BAS vs ECE \& AURC}
We provide simple synthetic examples to illustrate how BAS can diverge from standard uncertainty metrics such as ECE and AURC. These examples highlight that calibration and ranking alone may not capture decision-level reliability under abstention.

\paragraph{Identical ECE, different BAS.}
Consider two models evaluated on four examples with correctness labels $Z = [1, 1, 0, 0]$. The models produce the following confidence values:
\begin{center}
\begin{tabular}{lcccc}
Example & 1 & 2 & 3 & 4 \\
\midrule
$Z$ & 1 & 1 & 0 & 0 \\
Model A ($s_A$)\hspace{3.0mm} & 0.7 & 0.7 & 0.3 & 0.3 \\
Model B ($s_B$) & 0.9 & 0.9 & 0.99 & 0.01 \\
\end{tabular}
\end{center}
Using the unbinned calibration error $\mathrm{ECE} = \frac{1}{N} \sum_i |s_i - Z_i|$, both models have identical ECE:
\[
\mathrm{ECE}_A = \mathrm{ECE}_B = 0.30.
\]
However, their BAS values differ significantly. Using Eq.~\ref{eq:bas_closed_form}~\&~\ref{eq:BAS_metric}, we obtain:
\[
\mathrm{BAS}_A \approx 0.32, \qquad \mathrm{BAS}_B \approx -0.45.
\]
The difference arises from a single highly overconfident error ($s=0.99$, $Z=0$) in Model B, which incurs a large negative penalty under BAS due to the logarithmic term. This example shows that ECE treats overconfidence and underconfidence symmetrically, whereas BAS penalizes catastrophic overconfidence more strongly.

\paragraph{Identical AURC, different BAS.}
Consider two models evaluated on six examples with $Z = [0, 1, 1, 1, 0, 0]$. Both models produce confidence values with identical ranking:
\begin{center}
\begin{tabular}{lcccccc}
Example & 1 & 2 & 3 & 4 & 5 & 6 \\
\midrule
$Z$ & 0 & 1 & 1 & 1 & 0 & 0 \\
Model A ($s_A$)\hspace{3.0mm} & 0.90 & 0.80 & 0.70 & 0.40 & 0.30 & 0.20 \\
Model B ($s_B$) & 0.99 & 0.85 & 0.75 & 0.45 & 0.35 & 0.25 \\
\end{tabular}
\end{center}
Since AURC depends only on the relative ranking/ordering of confidence values, both models achieve identical AURC. However, BAS differs:
\[
\mathrm{BAS}_A \approx 0.07, \qquad \mathrm{BAS}_B \approx -0.28.
\]
The discrepancy is again driven by extreme overconfidence. In Model B, the highest-confidence example is incorrect ($s=0.99$, $Z=0$), resulting in a large negative contribution. This penalty is not captured by AURC, which is invariant to monotonic transformations of the confidence values.

\paragraph{Summary.}
These examples illustrate that: (1) ECE evaluates calibration but does not account for decision consequences or asymmetric risk. (2) AURC evaluates ranking but ignores the magnitude of confidence. (3) BAS evaluates decision utility, capturing both calibration, ranking, and the cost of overconfident errors. As a result, BAS can distinguish models that appear equivalent under these standard metrics, particularly in settings where rare but highly confident errors dominate decision risk.

\subsection{Illustrative Examples: BAS vs Log Loss \& Brier Score}
\label{sec:appendix_metrics_comparison_logloss}

We provide two simple synthetic examples to illustrate how models can have identical log loss and Brier score but different BAS.

\paragraph{Example 1.}
\begin{center}
\begin{tabular}{lcccc}
Example & 1 & 2 & 3 & 4 \\
\midrule
$Z$ & 1 & 1 & 0 & 0 \\
Model A ($s_A$)\hspace{3.0mm} & 0.90 & 0.01 & 0.10 & 0.10 \\
Model B ($s_B$) & 0.90 & 0.90 & 0.10 & 0.99 \\
\end{tabular}
\end{center}
Both models have identical log loss and Brier score:
\[
\mathrm{LogLoss}_A = \mathrm{LogLoss}_B \approx 1.23, \qquad \mathrm{Brier}_A = \mathrm{Brier}_B \approx 0.25.
\]
However, their BAS values differ:
\[
\mathrm{BAS}_A \approx 0.22, \qquad \mathrm{BAS}_B \approx -0.46.
\]

\paragraph{Example 2.}
\begin{center}
\begin{tabular}{lcccc}
Example & 1 & 2 \\
\midrule
$Z$ & 1 & 0\\
Model A ($s_A$)\hspace{3.0mm} & 0.999 & 0.999 \\
Model B ($s_B$) & 0.001 & 0.001 \\
\end{tabular}
\end{center}
Both models have identical log loss and Brier score:
\[
\mathrm{LogLoss}_A = \mathrm{LogLoss}_B \approx 3.45, \qquad \mathrm{Brier}_A = \mathrm{Brier}_B \approx 0.50.
\]
However, their BAS values differ:
\[
\mathrm{BAS}_A \approx -2.45, \qquad \mathrm{BAS}_B \approx 0.0005.
\]

\paragraph{Summary.}
These examples highlight that log loss and Brier score are invariant to permutations of confidence values across examples, treating underconfidence and overconfidence symmetrically. In contrast, BAS imposes an asymmetric penalty that strongly penalizes overconfident errors, which can lead to fundamentally different evaluations.

\subsection{Experimental Results}

Figure~\ref{fig:bas_vs_metrics} shows the relationship between BAS and the standard reliability metrics (ECE, AURC) across all model-dataset pairs, as discussed in Section~\ref{section:results_metrics_comp}.

\begin{figure}[t]
    \centering
    \includegraphics[clip, trim=12.0cm 0.0cm 0.0cm 0.0cm, width=0.90\textwidth]{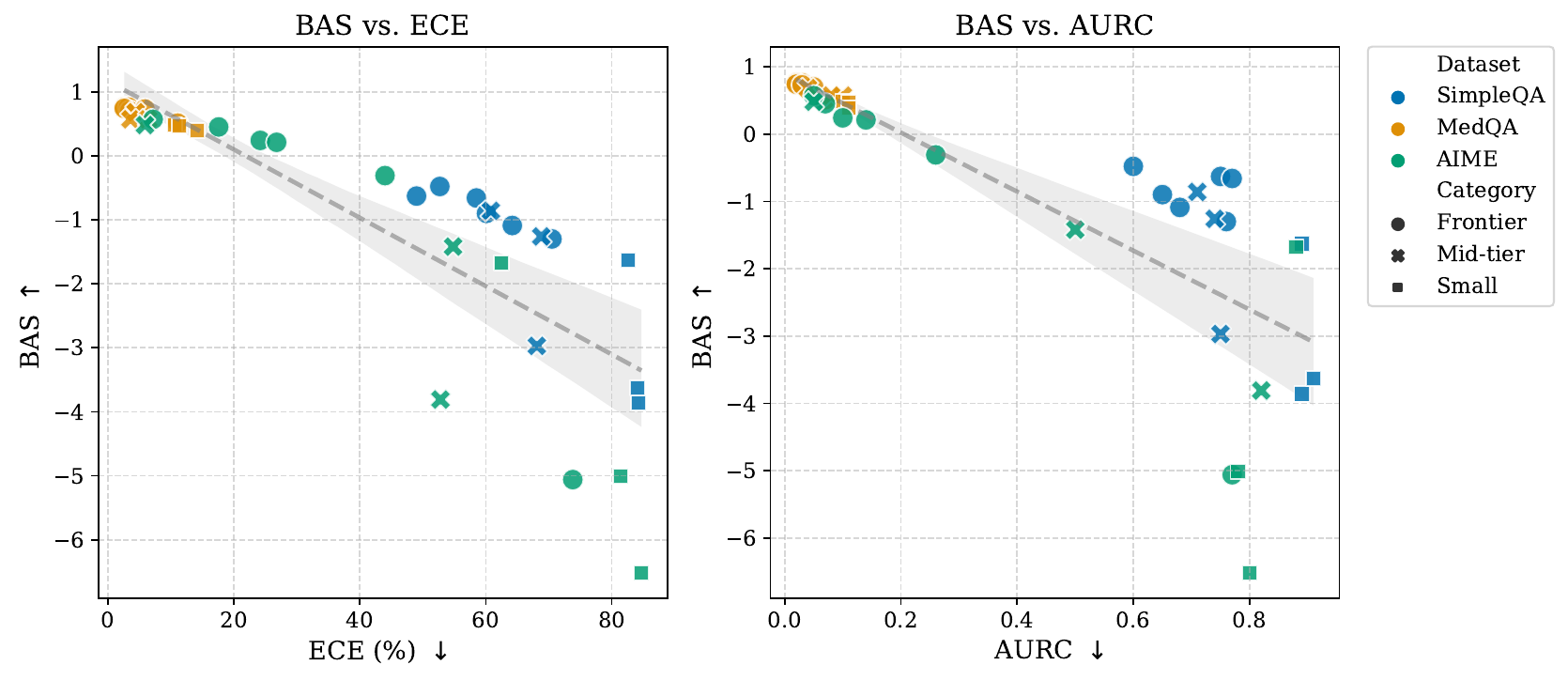}\vspace{-1.0mm}
    \caption{Relationship between BAS and the standard reliability metrics (ECE, AURC) across all model-dataset pairs. While BAS is correlated with ECE and AURC, several points deviate substantially from the trend, indicating that models with similar confidence calibration or ranking can exhibit highly different decision-level reliability.}
    \label{fig:bas_vs_metrics}
\end{figure}

Figure~\ref{fig:confidence_histograms} shows the distribution of predicted confidence values for Llama 3.3 and Mistral (M) on AIME and SimpleQA, as discussed in Section~\ref{section:results_metrics_comp}.

\begin{figure}[t]
    \centering
    \includegraphics[width=1.0\textwidth]{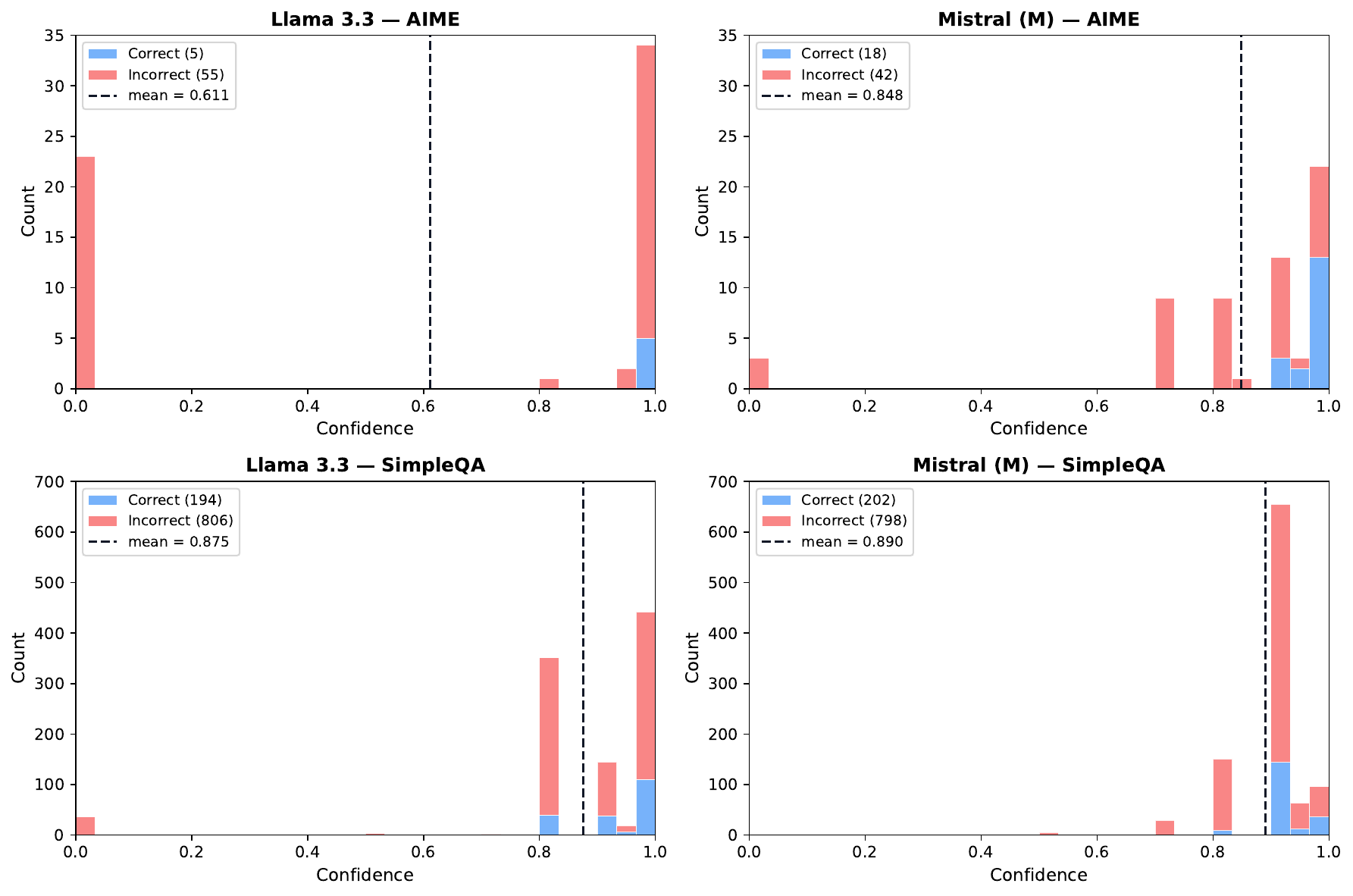}\vspace{-1.0mm}
    \caption{Distribution of predicted confidence values for Llama 3.3 and Mistral (M) on AIME and SimpleQA, separated by correctness. While both models exhibit high-confidence predictions, Llama 3.3 assigns extremely high confidence to a larger number of incorrect answers. These rare but highly overconfident errors are strongly penalized by BAS but contribute only modestly to ECE, explaining the observed discrepancies between the metrics.}
    \label{fig:confidence_histograms}
\end{figure}

Figure~\ref{fig:bas_vs_logloss} shows the relationship between BAS and the proper scoring rule log loss across all model-dataset pairs, as discussed in Section~\ref{section:results_metrics_comp}.

\begin{figure}[t]
    \centering
    \includegraphics[width=1.0\textwidth]{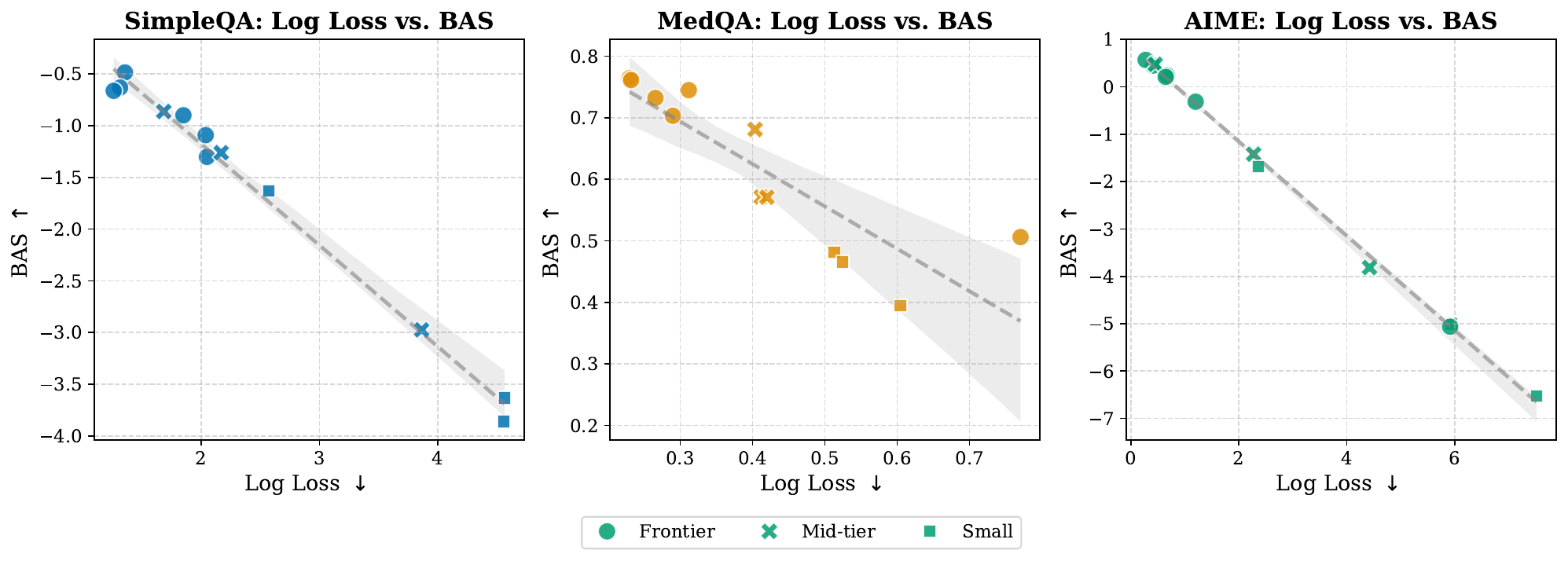}\vspace{-1.0mm}
    \caption{Relationship between BAS and the proper scoring rule log loss across all model-dataset pairs. The two metrics are highly correlated in practice, reflecting their shared sensitivity to high-confidence errors and the clear tendency of modern LLMs to exhibit overconfidence rather than underconfidence. However, BAS differs structurally by imposing an asymmetric penalty that strongly prioritizes avoiding overconfident errors, enabling it to distinguish between models that trade off overconfidence with underconfidence.}
    \label{fig:bas_vs_logloss}
\end{figure}

\section{Additional Experimental Details}

\subsection{Black-box Evaluation Protocol}
BAS requires only text-level access to the model and is compatible with both open-weight and proprietary systems. For each benchmark example, the model is queried once to produce a response $R$ and an associated confidence value $s$. During evaluation, we compute the per-example score using Eq.~\ref{eq:bas_closed_form} (or the numerical weighted variant). To ensure numerical stability, we clip $s$ to $\bar{s} = \min(s, 1-\epsilon)$ with a fixed $\epsilon=10^{-4}$ before computing Eq.~\ref{eq:bas_closed_form} or any weighted integrals.

All evaluations are conducted using deterministic decoding with temperature set to $0$. For each benchmark, we use a structured system prompt instructing the model to provide a single final answer along with a confidence score. For example, in MedQA's USMLE-style multiple-choice setting, the model is prompted to select a single option (A--D).

Per-example BAS utility is computed using Eq.~\ref{eq:bas_closed_form}, and results are averaged over the dataset. For AIME~2024/25, correctness is determined via exact numerical match against the ground-truth answer. For SimpleQA, correctness is evaluated using GPT-4o as an LLM judge \citep{gu2024survey}. The judge is provided both the model's answer and the reference answer, reducing the task to a constrained semantic verification problem. To assess potential bias, we conduct a human validation study on a random sample of 200 model--judge pairs and observe 100\% agreement between human annotations and the judge's outputs. The resulting correctness label $Z \in \{0,1\}$ is used directly in the BAS computation.

Full prompts and additional implementation details are provided in Appendix~\ref{sec:implementation_details}.

\subsection{Calibration and Evaluation Splits}
We utilize strictly disjoint sets for calibration and testing. For SimpleQA, we utilize a 1,000/1,000 split for calibration and evaluation, respectively. For MedQA, we utilize a calibration set of 1,273 examples and evaluate on a separate test set of 11,273 examples. Calibration parameters (i.e., the isotonic regression mapping $f$) are computed once on the calibration split and held fixed during testing.

\begin{figure}[t]
    \centering
    \includegraphics[width=\textwidth]{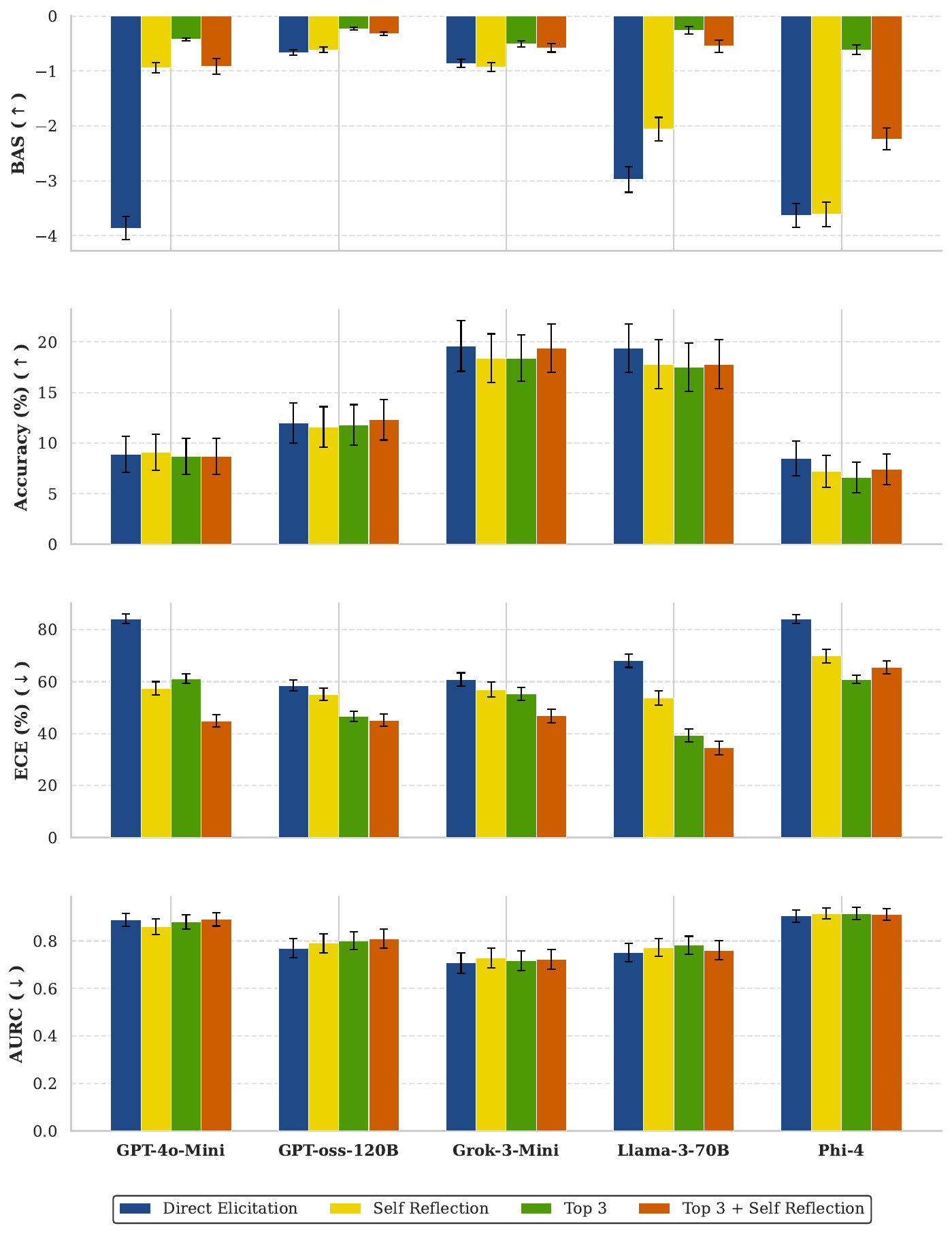}
\caption{Comparison of confidence elicitation methods on SimpleQA across BAS, accuracy, ECE, and AURC. Results shown for the same subset of representative models as in Table~\ref{tab:conf_elicitation_results}. Top-$k$ consistently improves BAS by reducing overconfident errors, while accuracy and AURC remain largely unchanged across methods. ECE improves when moving away from direct elicitation, indicating better calibration.}
 \label{fig:conf_elicitation_results}
\end{figure}

\section{Additional Results}
\begin{table}[h]
\centering\small
\begin{tabular}{l c c}
\toprule
Val set size & BAS & ECE (\%) $\downarrow$ \\
\midrule
1000 & $-0.000 \pm 0.006$ & $6.0 \pm 1.4$ \\
500  & $0.002 \pm 0.004$  & $6.0 \pm 1.4$ \\
250  & $0.000 \pm 0.004$  & $6.7 \pm 1.4$ \\
100  & $-0.005 \pm 0.004$ & $9.3 \pm 1.5$ \\
50   & $-0.000 \pm 0.003$ & $6.3 \pm 1.5$ \\
25   & $-0.016 \pm 0.005$ & $13.7 \pm 1.6$ \\
10   & $-0.072 \pm 0.010$ & $23.9 \pm 1.9$ \\
\bottomrule
\end{tabular}
\caption{Post-hoc calibration (isotonic regression) ablation results for GPT-oss-120B on SimpleQA. Validation set: 10 - 1,000 samples, test set: 1,000 samples.}
\label{tab:calibration_ablation}
\end{table}

Table~\ref{tab:calibration_ablation} presents an ablation study on the sample efficiency of isotonic regression for post-hoc calibration. Importantly, the results demonstrate that decision-level reliability is highly sample-efficient to recover: the model achieves near-optimal BAS and low ECE with as few as 50 to 100 calibration samples, with significant performance degradation occurring only in extremely low-data regimes (25 samples or fewer).

Figure~\ref{fig:conf_elicitation_results} reports the full results for Section~\ref{section:results_conf_elicitation}, comparing four confidence elicitation methods on SimpleQA in terms of BAS, accuracy, ECE, and AURC.

\section{Extended Methodology: Post-hoc Confidence Calibration}
\label{sec:appendix_calibration_details}

\begin{figure}[h]
    \centering
    \includegraphics[width=\textwidth]{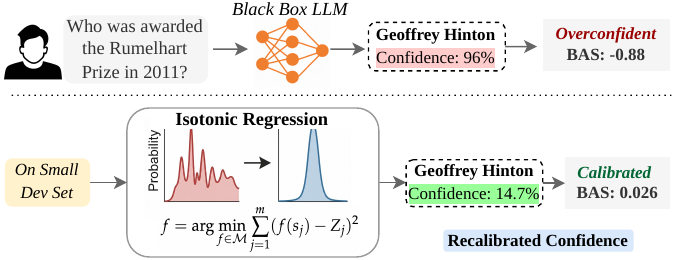}
\caption{
Effect of post-hoc calibration on model confidence. Raw model confidence values are transformed using isotonic regression to better align with empirical accuracy, leading to improved decision-theoretic evaluation under BAS.
}   \label{fig:calibration_diagram}
\end{figure}

To study how calibration affects decision-theoretic evaluation, we apply post-hoc calibration to the model's confidence and measure its impact on BAS and related metrics.

\paragraph{Calibration procedure.}
For each prompt $x$, the model produces a response $y$ and a scalar confidence $s \in [0,1)$. We transform this raw confidence into a calibrated estimate $\hat{p}$ using a mapping $f: s \mapsto \hat{p}$ learned on a held-out validation set $\mathcal{D}_{\text{val}} = \{(s_j, Z_j)\}_{j=1}^m$, where $Z_j \in \{0,1\}$ denotes correctness.

Following standard practice, we use \textit{isotonic regression}, a non-parametric and monotonic calibration method:
\begin{equation}
\hat{p} = f(s), \quad 
f = \arg\min_{f \in \mathcal{M}} \sum_{j=1}^m (f(s_j) - Z_j)^2,
\end{equation}
where $\mathcal{M}$ is the set of non-decreasing functions. This approach preserves ranking while correcting systematic miscalibration, particularly overconfidence.

\paragraph{Effect on BAS.}
We evaluate BAS using the calibrated confidence $\hat{p}$ in place of the raw value $s$. Using the closed-form expression from Eq.~\ref{eq:bas_closed_form}, the per-example utility becomes:
\begin{equation}
U(\hat{p}, Z) =
\begin{cases}
\hat{p}, & Z = 1, \\
\hat{p} + \ln(1 - \hat{p}), & Z = 0,
\end{cases}
\end{equation}
with $\hat{p}$ clipped for numerical stability. Because BAS penalizes overconfident errors through the logarithmic term, calibration directly affects the magnitude of negative utility assigned to incorrect predictions. This post-hoc calibration procedure is illustrated in Figure~\ref{fig:calibration_diagram}.

\section{Implementation Details}
\label{sec:implementation_details}

\begin{tcolorbox}[
    colback=gray!5,
    colframe=gray!75!black,
    title=BAS System Prompt (MedQA Example),
    fonttitle=\bfseries,
    boxrule=0.5pt,
    arc=2pt,
    left=6pt,
    right=6pt,
    top=6pt,
    bottom=6pt
]

\ttfamily
You are a medical expert assisting with USMLE-style questions. 
Analyze the clinical case and select the most appropriate option.

At the end of your response, provide your final decision in a 
block starting with '\#\#\# FINAL DECISION'.

Provide the single letter corresponding to the correct option 
and a confidence score.

Format:
\#\#\# FINAL DECISION
Answer: [Letter A-D]
Confidence: [0-1 representing the probability your answer is correct]
\end{tcolorbox}
\begin{tcolorbox}[
    colback=gray!5,
    colframe=gray!75!black,
    title=LLM as a Judge System Prompt (SimpleQA),
    fonttitle=\bfseries,
    boxrule=0.5pt,
    arc=2pt,
    left=6pt,
    right=6pt,
    top=6pt,
    bottom=6pt
]

\ttfamily
You are an automated grading system. Your task is to evaluate whether the 
Model Answer provided is factually consistent with the Ground Truth answer.

Question: \{question\}

Ground Truth: \{gt\}

Model Answer: \{model\_ans\}

Instructions:
1. Compare the core factual claim of the Model Answer to the Ground Truth.
2. If they are semantically identical or convey the same factual information, respond with 'CORRECT'.
3. If they contradict or if the Model Answer contains different factual information, respond with 'INCORRECT'.
Output only 'CORRECT' or 'INCORRECT'.

\end{tcolorbox}

\subsection{Prompts for Benchmarking Verbalized Confidence}
\label{sec:prompt_appendix}
\begin{tcolorbox}[
    title=System Prompt for Direct Elicitation,
    colback=gray!5,
    colframe=gray!75!black,
    fonttitle=\bfseries
]
\ttfamily
You are a short-form question answering system.

Answer the question using ONLY the final answer.
The answer must be as short as possible (a word, number, or short phrase).

At the very end, provide your final output in this exact format:
\begin{verbatim}
### FINAL DECISION
Answer: <short answer only>
Confidence: <number between 0 and 1>
\end{verbatim}
\end{tcolorbox}

\begin{tcolorbox}[
    title=System Prompt for Top-k,
    colback=gray!5,
    colframe=gray!75!black,
    fonttitle=\bfseries
]
\ttfamily
You are a short-form question answering system.

Provide your top $k$ most likely answers to the question. 
For each answer, provide a confidence score (probability) between 0 and 1. 
The sum of all confidence scores must equal 1.0.

The answers must be as short as possible (a word, number, or short phrase).

At the very end, provide your final output in this exact format:
\begin{verbatim}
### FINAL DECISION
1. Answer: <answer 1>, Confidence: <score 1>
2. Answer: <answer 2>, Confidence: <score 2>
...
\end{verbatim}
\end{tcolorbox}
\begin{tcolorbox}[
    title=System Prompt for Self-Reflection,
    colback=gray!5,
    colframe=gray!75!black,
    fonttitle=\bfseries
]
\ttfamily
\textbf{Step 1: Generation}

You are a short-form question answering system.
Answer the question using ONLY the final answer.
The answer must be as short as possible (a word, number, or short phrase).

\medskip
\textbf{Step 2: Reflection}

You are an expert evaluator. Given a question and a proposed answer, 
assess the probability that the answer is factually correct.

Provide your confidence as a single number between 0 and 1.

Format your output exactly as:
\begin{verbatim}
Confidence: <number>
\end{verbatim}
\end{tcolorbox}
\begin{tcolorbox}[
    title=System Prompt for Top-k + Self-Reflection,
    colback=gray!5,
    colframe=gray!75!black,
    fonttitle=\bfseries
]
\ttfamily
\textbf{Step 1: Candidate Generation}

You are a short-form question answering system.
Provide exactly $k$ distinct candidate answers to the question, 
ordered from most likely to least likely.
Answers must be as short as possible (a word, number, or short phrase).

\medskip
\textbf{Step 2: Probability Distribution Reflection}

You are an expert evaluator. Below is a question and $k$ candidate answers.
Assign a probability (between 0 and 1) to each answer based on how likely it is to be correct.
The sum of all probabilities must equal 1.0.

Format your final output exactly as follows:
\begin{verbatim}
### FINAL DECISION
1. Answer: <answer 1>, Confidence: <score 1>
2. Answer: <answer 2>, Confidence: <score 2>
...
\end{verbatim}
\end{tcolorbox}

\section{Example Code Snippet}
\lstdefinestyle{basstyle}{
    backgroundcolor=\color{backcolour},   
    commentstyle=\color{codegreen},
    keywordstyle=\color{magenta},
    numberstyle=\tiny\color{codegray},
    stringstyle=\color{codepurple},
    basicstyle=\ttfamily\small,
    breakatwhitespace=false,         
    breaklines=true,                 
    captionpos=b,                    
    keepspaces=true,                 
    numbers=left,                    
    numbersep=5pt,                  
    showspaces=false,                
    showstringspaces=false,
    showtabs=false,                  
    tabsize=2
}

\begin{lstlisting}[language=Python, style=basstyle, caption={Example usage of the \texttt{bas-eval} package for measuring behavioral alignment. Plug-and-play metric.}]
import pandas as pd
from bas_eval import bas_score, BASReport

# 1. Load your model results (e.g., from MedQA or SimpleQA)
# Dataset should contain 'is_correct' (bool) and 'confidence' (float [0,1])
df = pd.read_csv("model_results.csv")

# 2. Simple API: Compute standard BAS (Uniform risk prior)
# Reflects the expected utility aggregated over all risk thresholds
score = bas_score(df['is_correct'], df['confidence'])
print(f"Mean BAS: {score:.4f}")

# 3. Advanced API: Generate a full Alignment Report
# Automatically computes Uniform, Linear, and Quadratic risk profiles
report = BASReport(df['is_correct'], df['confidence'])

# Display a summary table for the paper
report.print_summary()

# 4. Access safety-critical metrics specifically
# Penalizes high-confidence hallucinations more severely
safety_score = report.weighted_score(prior='quadratic')
print(f"Safety-Critical BAS (w(t)=3t^2): {safety_score:.4f}")
\end{lstlisting}

\end{document}